\documentclass[acmsmall]{acmart}

\usepackage{algorithmic}
\usepackage{graphicx}
\usepackage{textcomp}
\usepackage{caption}
\usepackage{xcolor}
\usepackage{subcaption}
\usepackage{booktabs}
\usepackage{wrapfig}
\usepackage{adjustbox}
\usepackage[pagewise]{lineno}
\usepackage{multirow}

\usepackage{color, soul}
\usepackage[most]{tcolorbox}
\newtcolorbox{highlighted}{colback=yellow,breakable}

\usepackage[linesnumbered,ruled,vlined]{algorithm2e}
\SetKwInput{KwInput}{Input}
\SetKwInput{KwOutput}{Output}

\newtheorem{definition}{Definition}

\definecolor{tamercolor}{rgb}{0.5, 0.0, 0.13}
\definecolor{forest}{rgb}{0.25, 0.74, 0.0}

\newcommand{\NAMEA}{FIDDLE}

\AtBeginDocument{%
  }

\setcopyright{acmlicensed}
\copyrightyear{XXXX}
\acmYear{XXXX}
\acmDOI{XXXXXXX.XXXXXXX}

\acmJournal{TQC}
\acmVolume{37}
\acmNumber{4}
\acmArticle{111}
\acmMonth{8}

\begin{document}

\title{\NAMEA: Reinforcement Learning for Quantum Fidelity Enhancement}

\author{Hoang M. Ngo}
\orcid{1234-5678-9012}
\affiliation{%
  \institution{Department of Computer and Information Science and Engineering, University of Florida}
  \city{Gainesville}
  \state{Florida}
  \country{USA}
}
\author{Tamer Kahveci}
\orcid{1234-5678-9012}
\affiliation{%
  \institution{Department of Computer and Information Science and Engineering, University of Florida}
  \city{Gainesville}
  \state{Florida}
  \country{USA}
}
\author{My T. Thai}
\orcid{1234-5678-9012}
\affiliation{%
  \institution{Department of Computer and Information Science and Engineering, University of Florida}
  \city{Gainesville}
  \state{Florida}
  \country{USA}
}

\renewcommand{\shortauthors}{Hoang Ngo et al.}

\begin{abstract}

Quantum computing has the potential to revolutionize fields like quantum optimization and quantum machine learning. However, current quantum devices are hindered by noise, reducing their reliability. A key challenge in gate-based quantum computing is improving the reliability of quantum circuits, measured by process fidelity, during the transpilation process, particularly in the routing stage. In this paper, we address the Fidelity Maximization in Routing Stage (FMRS) problem by introducing \mbox{\NAMEA}, a novel learning framework comprising two modules: a Gaussian Process-based surrogate model to estimate process fidelity with limited training samples and a reinforcement learning module to optimize routing. Our approach is the first to directly maximize process fidelity, outperforming traditional methods that rely on indirect metrics such as circuit depth or gate count. We rigorously evaluate \mbox{\NAMEA} by comparing it with state-of-the-art fidelity estimation techniques and routing optimization methods. The results demonstrate that our proposed surrogate model is able to provide a better estimation on the process fidelity compared to existing learning techniques, and our end-to-end framework significantly improves the process fidelity of quantum circuits across various noise models.

\end{abstract}

\keywords{Quatum computing, quantum circuits, optimization, fidelity}

\maketitle

\section{Introduction}
\label{sec:introduction}
Despite its tremendous potential, current quantum computing faces a fundamental challenge: reliability. As noted in~\cite{Preskill2018}, we are in the Noisy Intermediate-Scale Quantum (NISQ) era, where quantum devices are highly susceptible to noise and decoherence. This vulnerability significantly reduces the reliability of NISQ systems as the scale and complexity of computations increase. Consequently, this limitation restricts the scalability of quantum computers for practical applications and remains one of the most significant obstacles to realizing the full promise of quantum computing. 

Gate-based quantum computing, as known as universal quantum computing, is one of the most prominent quantum paradigms. It is used to drive quantum advancements in diverse applications, such as quantum optimization~\cite{Farhi2014,Mohanty2023,Azad2023,Golden2023} and quantum machine learning ~\cite{schuld2019,Biamonte2017,Allcock2020, Ma2021, Jerbi2021QuantumML}.
In the quantum gate-based paradigm, computations required by quantum algorithms are represented as quantum circuits, which are generated through a transpilation process. Therefore, enhancing the reliability of quantum circuits during transpilation is crucial for advancing the gate-based quantum paradigm to the next stage of development.

In order to improve the reliability of quantum circuits, we need to address two critical tasks: (1) defining metrics to quantify reliability and (2) developing techniques to improve reliability based on those metrics. For the first task, several reliability metrics have been proposed, including Probability of Successful Trial (PST), Hellinger Fidelity (HF), and Process Fidelity (PF)~\cite{Resch2021}. Specifically, PST quantifies the probability of obtaining a specific desired outcome after running the quantum process, while HF measures the similarity between the probability distributions of the ideal state and the actual probabilities measured in the $Z$-basis. However, both metrics are limited to computational $Z$-basis measurements. That leads to a failure in distinguishing states which yield identical $Z$-basis distributions, like $\frac{1}{\sqrt{2}}|0\rangle + \frac{1}{\sqrt{2}}|1\rangle$ and $\frac{1}{\sqrt{2}}|0\rangle - \frac{1}{\sqrt{2}}|1\rangle$. In contrast, PF evaluates the overlap between ideal and actual quantum states across $4^N$ bases (i.e., $N$ is the number of qubits), offering a more reliable metric. 
Exact calculation of PF is QSZK-hard~\mbox{\cite{Watrous2002, Wang2023}}, so methods for calculating PF like direct fidelity estimation~\mbox{\cite{Flammia_2011}} are computationally expensive. To accelerate PF estimation, learning-based techniques, which are trained on computed fidelity samples to predict the fidelity of unseen circuits, have been investigated. Notable approaches include deep neural network-based method~\mbox{\cite{Zhang2021}}, and graph neural network-based method~\mbox{\cite{Saravanan_2024}}. However, these approaches require a large number of training samples including circuits and their PF values. Thus, learning-based techniques in the existing literature are impractical when the sample availability is limited due to the time-consuming calculation of PF values.

For the second task, due to the challenge of accurately calculating reliable fidelity metrics like PF, most methods optimize factors indirectly related to circuit reliability, such as circuit depth and gate count. Circuit optimization typically occurs during the transpilation process, which involves three major stages: layout (mapping logical to hardware qubits), routing (ordering gates and adding SWAPs), and translation (converting to native gates). For example, works in~\mbox{\cite{Zulehner2018,Ostaszewski2021,Pozzi2022}} focus on routing to minimize gate count and depth, while~\mbox{\cite{Fan2022}} combines layout and routing optimization. Alam et al.~\mbox{\cite{Alam2020}} focused on routing circuits for the Quantum Approximate Optimization Algorithm (QAOA). The circuits used for QAOA contains mainly commuting gates (i.e., $R_{zz}$) which can be reordered without affecting the outcome. As a result, routing for QAOA circuits requires additional consideration of gate ordering to optimize circuit reliability. Unlike depth and gate count optimization, Saravanan et al.~\mbox{\cite{Saravanan_2024}} directly optimize circuit reliability using PST and HF. However, as mentioned above, PST and HF do not fully capture the actual reliability of quantum circuits like PF. Thus, this approach may provide only a partial improvement in the reliability.

In this work, we study the problem of optimizing the routing stage in the transpilation process to maximize the PF of the resulting quantum circuit. To our knowledge, this is the first work to directly optimize the PF of quantum circuits. First of all, we introduce the concept of gate sequence as the output of the routing stage and establish its connection to the circuit which is the end-to-end solution of the transpilation process. This linkage allows us to evaluate the efficiency of gate sequences based on circuit fidelity. Based on this, we formally define the Fidelity Maximization in Routing Stage (FMRS) problem. This problem aims to find a valid gate sequence in the routing stage to maximize the PF of the resulting circuit.

Next, we propose a novel learning framework, named \mbox{\NAMEA}, which seeks for the optimal gate sequence in order to maximize the circuit fidelity. Our proposed framework comprises two key learning modules which aim to address two aforementioned critical tasks for improving circuit reliability. For the first module, in order to reduce the computational cost of PF estimation, we propose a Gaussian Process (GP) -based surrogate model to estimate the PF of quantum circuits. Different with other deep learning techniques requiring extensive training samples, our GP-based surrogate model can excel with limited training samples by learning the underlying data distribution. Then, we propose a reinforcement learning (RL) module trained on PF values predicted by the surrogate model to identify optimal gate sequences. By directly optimizing PF, the RL module is able to provide more effective solutions than traditional methods relying on indirect metrics like circuit depth or gate count. 

Finally, we rigorously evaluate our proposed framework by comparing it against state-of-the-art methods for fidelity estimation and routing optimization. We assess the accuracy of our GP-based surrogate model by measuring its deviation from the direct fidelity estimation in~\mbox{\cite{Flammia_2011}}, which closely approximates the exact PF. We also compare against the latest learning-based fidelity estimation technique proposed in~\mbox{\cite{Saravanan_2024}} to illustrate the superior of our proposed learning technique with limited training samples. Our end-to-end framework, \mbox{\NAMEA}, is then evaluated against VIC~\mbox{\cite{Alam2020}}, a state-of-the-art routing method for QAOA circuits, and Qiskit~\mbox{\cite{qiskit}}, a widely used off-the-shelf transpilation framework, across five different noise models. The results demonstrate that our approach significantly improves the PF of quantum circuits under every noise model.

\textbf{Organization.} The rest of the paper is structured as follows. Sections~\ref{sec:relatedworks} and \ref{sec:preliminaries} summarize the related work and introduce the definition of minor embedding
and an overview of reinforcement learning, respectively. Our problem is rigorously defined in Section \ref{sec:problem_definition}, where our method and its theoretical
analysis are described in Section~\ref{sec:proposedapproach}. Section~\ref{sec:experiments} presents our
experimental results. Finally, section~\ref{sec:conclusion} concludes the paper.



\section{Related works}
\label{sec:relatedworks}
In this section, we present illustrative works in literature on two critical tasks for improving the reliability of quantum circuits.

For the first task, PF is widely regarded as the standard for quantifying the reliability of quantum processes. 
However, PF is QSZK-hard to compute exactly~\cite{Watrous2002, Wang2023}, so there are no efficient methods for its estimation in general scenarios. One straightforward approach involves obtaining a complete classical description of quantum states as density matrices through quantum state tomography~\cite{Gross2010, Cramer2010}. However, this method requires exponentially growing resources as the scale of the quantum system increases. For specific cases where the ideal state is pure, several techniques have been proposed for estimating PF. One such method, called entanglement witnesses\cite{Tokunaga2006, Otfried2007}, computes fidelity with a limited number of measurements, but it only works for certain pure quantum states. This limitation was addressed in\cite{Flammia_2011}, which introduced a direct fidelity estimation method for arbitrary pure states. More recently, a deep neural network-based approach has been developed for estimating PF~\cite{Zhang2021}. However, the effectiveness of this method is heavily reliant on the availability of a large number of training samples (i.e., circuits and their corresponding PF values). As a result, this approach struggles to train effectively on large quantum circuits, where obtaining the PF values is time-consuming.

For the second task, the paper~\cite{Zulehner2018} proposes a two-step routing method to reduce circuit depth on IBM-QX hardware. This technique first partitions the quantum circuit into layers, ensuring that all gates within a layer are compatible with the hardware's topology through an appropriate layout. Then, an $A^*$-based algorithm is employed for each layer to determine the minimal number of SWAP gates required to adjust the layout for alignment with the subsequent layer. This approach demonstrates notable improvements in both runtime performance and circuit depth compared to IBM’s native solution. Machine learning has proven to be a powerful tool for addressing these challenges. For instance, a machine learning approach has been developed for the initial placement of quantum circuits (i.e., the layout optimization) by implementing automatic feature selection. This technique focuses on significant subcircuits using a configurable Gaussian function. In contrast, reinforcement learning has shown its effectiveness in tackling the gate placement problem (i.e., the routing pass). Specifically, the works~\cite{Ostaszewski2021,Pozzi2022} utilize a modified deep Q-learning framework to address the routing pass, aiming to minimize the depth of the resulting circuit. Here, the routing problem is modeled as a Markov decision process, where the state represents the quantum circuit, the action corresponds to gate placement within the circuit, and the reward reflects metrics such as circuit depth or gate count, depending on the optimization objective. The paper~\cite{Fan2022} presents a hybrid approach combining the layout and routing stages by framing the circuit placement problem as a bi-level optimization task. In this framework, the upper-level problem optimizes the initial qubit mapping, while the lower-level problem focuses on minimizing the number of required SWAP gates for a given upper-layer's solution, formulated as a combinatorial optimization problem. To address this, the authors propose an evolutionary algorithm (EA) for optimizing the upper-level mapping and employ policy-based deep reinforcement learning to solve the lower-level routing problem. 

The work in ~\cite{Alam2020} studies the optimization of the routing stage for circuits derived from the Quantum Approximate Optimization Algorithm (QAOA). In QAOA circuits, operations (i.e., $R_{ZZ}$ gate)s are commutative. It means that their execution order can be freely swapped without affecting the final state of the quantum circuit. This commutativity, while beneficial for circuit design, also makes routing more complex, since a large number of potential gate orderings must be considered. To address this, the authors propose several techniques to rank the execution order of $R_{ZZ}$ gates in a way that minimizes the circuit depth or gate count. This work is regarded as one of the state-of-the-art routing techniques for QAOA circuits.

Unlike approaches that focus on circuit depth and gate count, the work by Saravanan et al.~\cite{Saravanan_2024} targets the optimization of a circuit's PST and HF. Their method employs a routing framework that integrates reinforcement learning (RL) with graph neural networks (GNN) to directly optimize these metrics. Specifically, the GNN model predicts the PST and HF of a circuit based on its structure, and these predicted values are used as reward signals in the RL-based routing framework. The GNN-based estimator in this approach benefits from the relative ease of computing PST and HF, allowing it to be trained on a large dataset. However, this advantage does not apply to PF, which is considerably more time-consuming to calculate. As a result, acquiring a sufficiently large training dataset for PF becomes impractical, restricting the scalability of their method when applied to large circuits. Additionally, the proposed RL framework has a drawback. If the selection of a gate does not lead to the formation of an executable layer, the state remains unchanged (i.e., $s_{t+1} = s_t$). This design can result in infinite loops during training, obstructing the convergence of the RL policy.

\section{Preliminaries}
\label{sec:preliminaries}
\subsection{General concepts in quantum computing}
Quantum computing relies on quantum bits (qubits), which are the quantum counterparts of classical bits. Unlike classical bits that are either 0 or 1, a qubit can be in a superposition of both states simultaneously. The state of a qubit is represented as "ket vector" $|\psi\rangle$ which is a linear combination of basis states $|0\rangle$ and $|1\rangle$ as: $$|\psi\rangle = \alpha_0 |0\rangle + \alpha_1 |1\rangle$$

In a system of $N$ qubits, the quantum state is represented by $2^N$ basis states. For instance, in a 2-qubit system, the state can be expressed as:
$$|\psi\rangle = \alpha_{00} |00\rangle + \alpha_{01} |01\rangle + \alpha_{10} |10\rangle + \alpha_{11} |11\rangle$$

Quantum operators, or quantum gates, are unitary transformations that describe the evolution of a closed quantum system. For example, Pauli gates, denoted as $X$, $Y$ and $Z$, perform state flips, or the Hadamard gate, denoted as $H$, creates superpositions. The transformation from the state $|\psi\rangle$ to the state $|\psi'\rangle$ by the operator $U$ can be presented as: $$|\psi'\rangle = U|\psi\rangle$$

The state of a closed quantum system can be observed by a collection $\{M_m\}$ of measurement operators. The index $m$ indicates a possible measurement outcome of the state. Specifically, by measuring the state $|\psi\rangle$ by $\{M_m\}$, the probability of $|\psi\rangle$ being $m$ can be calculated as: $$Pr(m) = \langle\psi|M_m^\dagger M_m|\psi\rangle$$


\subsection{Mixed state and quantum noise model}
Due to inherent instability, the true state (a.k.a \emph{pure state}) of a quantum system is often not fully known. In such cases, the state of a quantum system is described as a statistical mixture of different possible pure states, called \emph{mixed state}. Unlike pure states, which are represented by ket vectors, mixed states are represented by density matrices. Specifically, for a quantum system in a mixed state, which can be in one of pure states $|\psi_i\rangle$ with the probability $p_i$, the density matrix is given by: $$\rho = \sum_i p_i |\psi_i\rangle\langle\psi_i|$$

The evolution of mixed quantum states is described by quantum channels, which capture the transformations induced by one or more gates applied to a mixed state. For example, consider a unitary channel $\mathcal{U}$ corresponding to a single gate $U$, the evolution of the density matrix $\rho$ through the channel $\mathcal{U}$ can be expressed as:
$$\mathcal{U}(\rho) = \sum_i p_i (U|\psi_i\rangle)(\langle\psi_i|U^\dagger) = U\rho U^\dagger$$


The concept of mixed states is useful in modeling the noise in quantum systems. Under the presence of the environmental noise, the state of the quantum system can either remain unchanged or transform to other states with certain probabilities. Thus, quantum noise can be formally described using quantum channels, which are completely positive trace-preserving (CPTP) maps acting on the density matrices. Specifically, the output of a noise channel $\mathcal{E}$ on a density matrix $\rho$ is as follows:

$$
\mathcal{E}(\rho) = \sum_i K_i \rho K_i^\dagger
$$

where $\{K_i\}$ are the Kraus operators satisfying $\sum_i K_i^\dagger K_i = I$. This representation allows modeling complex noise processes by decomposing them into simpler, more manageable components.

Let's consider an example of the bit-flip noise channel where qubits can flip from $|0\rangle$ to $|1\rangle$ or vice versa with the probability $p$. The Kraus operators of this channel is $K_0 = \sqrt{1-p} I$ and $K_1 = \sqrt{p} X$, where $X$ is the bit flip (Pauli-X) operation.

\begin{figure*}[t]
  \centering
  \includegraphics[width=1\textwidth]{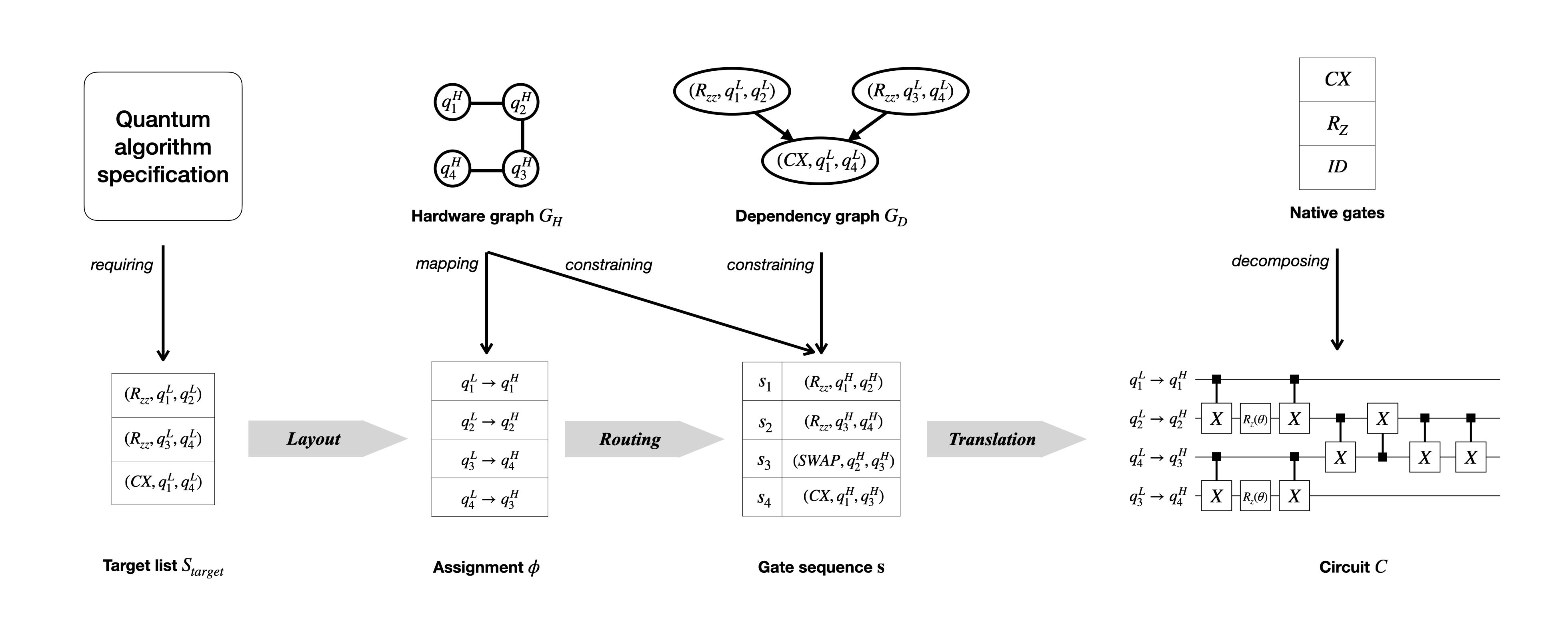}
\caption{The overview of transpilation process which involves three key stages: layout, routing, and translation. Starting with a target gate list required from a quantum algorithm, the layout stage maps logical qubits to physical qubits. Then, the routing stage determines the sequence of gates to be executed such that the gate sequence satisfies to dependencies and connectivity constraints (i.e., by inserting SWAP gates). Finally, the translation stage decomposes the target gates into native gates supported by the QPU. The result is an optimized quantum circuit.} 
\label{fig:example:overview}
\end{figure*}

\subsection{Transpilation process}

A quantum algorithm is initially expressed in terms of a list of target gates, which represent the logical operations required to solve a problem. To execute the algorithm on quantum hardware, they must undergo a process called \emph{transpilation}, which converts the abstract algorithm into a hardware-compatible quantum circuit. Transpilation consists of three key stages: layout, routing, and translation. The overview of the transpilation process is illustrated in Figure~\ref{fig:example:overview}.

In the layout stage, a mapping between the logical qubits defined in the quantum algorithm and the physical qubits available on the hardware (i.e., QPU) is established. Since the connectivity of physical qubits is sparse, this step must carefully optimize the assignment to minimize the overhead (i.e., the number of gates used, the number of layers used, or the circuit reliability) introduced by limited qubit connectivity. A poor initial layout can lead to excessive overheads in the routing and translation stages.

The routing stage in the transpilation process determines the sequence of gates which satisfies two key constraints: gate dependencies and hardware connectivity. Gate dependencies ensure that certain operations are executed in the correct order to maintain the algorithm's logical flow. Hardware connectivity constraints, specific to the quantum processing unit (QPU), require that two-qubit gates can only operate on physically adjacent qubits. To satisfy these constraints, the routing stage may insert SWAP gates to reposition hardware qubits. However, since additional SWAP gates can introduce noise, carefully selecting their number and placement is crucial to maintaining the circuit's fidelity.

Finally, in the translation stage, the high-level target gates are decomposed into native gates supported by the hardware. For example, a $R_{zz}$ gate can be decomposed into two $CX$ gates and one $R_{z}$ gate, as shown in Figure~\ref{fig:example:overview}. The output of the transpilation process is a quantum circuit $C$ which includes native gates along with their execution order and positions on the hardware. In the next section, we explore metrics to measure the reliability of a quantum circuit.

\subsection{Circuit Reliability Metrics}


Reliability metrics are crucial for evaluating the performance of quantum circuits. In this section, we introduce three commonly used metrics: Measurement Success Probability (PST), Hellinger Fidelity (HF), and Process Fidelity (PF). Then, we derive the circuit fidelity based on the PF.

These metrics primarily quantify the similarity between two quantum states. Specifically, when assessing circuit performance, metrics are used to measure the similarity between the ideal outcome of the circuit, typically represented as a pure state 
$\rho_\psi = |\psi\rangle\langle \psi|$, and the actual outcome of the circuit, represented as a mixed state $\rho$. High similarity indicates better performance and robustness against errors.

PST measures the total probability of obtaining non-zero amplitude states when performing measurements in the computational 
$Z$-basis. It is formally defined as~\mbox{\cite{Aktar2024}}:

$$PST(\rho_\psi, \rho) = \sum_{x\in \{0,1\}^n, \langle x|\rho_\psi|x\rangle \neq 0 } \langle x|\rho|x\rangle$$

On the other hand, HF quantifies the similarity between the probability distribution of the ideal quantum state and the experimentally observed distribution of the actual state, based on measurements in the computational 
$Z$-basis. It is calculated as~\mbox{\cite{Aktar2024}}:

$$HF(\rho_\psi, \rho) = \bigg[\sum_{x\in \{0,1\}^n } \sqrt{\langle x|\rho_\psi|x\rangle \cdot \langle x|\rho|x\rangle}\bigg]^2$$

Unlike PST and HF which quantify the similarity based on measurements on only the computational $Z$-basis, PF evaluates the overlap between the ideal and actual quantum states across all $4^N$ measurement bases, where $N$ is the number of qubits. Thus, among three metrics, PF is able to provide the most accurate similarity between two quantum states. The PF is calculated as:

$$
PF(\rho_\psi, \rho) = \left[\text{Tr} \left( \sqrt{\sqrt{\rho} \rho_\psi \sqrt{\rho}} \right) \right]^2 = \langle \psi|\rho|\psi \rangle
$$

For simplicity and consistency, from here, we will use the term "fidelity" to specifically refer to Process Fidelity (PF) throughout this paper. Using the fidelity between two quantum states, we can define the fidelity of a quantum circuit $C$ which is represented as a sequence of operations. In ideal conditions, this sequence can be expressed as:

$$U_C =U_1U_2\dots U_L$$

where $U_i$ indicates the operation applied at the layer i-\emph{th}. In each layer, multiple gates can be applied simultaneously on different qubits, representing parallel processing. Thus, each $U_i$ can be expressed as a tensor product of $l_i$ native gates that are executed in parallel within the layer i-\emph{th}. Specifically, we have 
$$U_i = U_{i,1} \otimes U_{i,2} \otimes \dots \otimes U_{i,l_i}$$
where each $U_{i,j}$ is a native gate acting on a specific qubit (1-qubit gate) or pair of qubits (2-qubit gate). 
We denote the quantum channel which represents $U_C$ as $\mathcal{U}_C$.

In a noisy environment specified by a noise model $\mathcal{E}$, each native gate $U_{i,j}$ is replaced by a quantum channel $\Bar{\mathcal{U}}_{i,j}$ that represents its imperfect execution under the noise model. Specifically, we have:
$$\bar{\mathcal{U}}_{i,j}(\rho) = \mathcal{E}(U_{i,j}|\rho\rangle \langle \rho|U_{i,j}^\dagger)$$

As a result, the noisy channel for each layer is then given by
$$\Bar{\mathcal{U}}_i = \Bar{\mathcal{U}}_{i,1} \otimes \Bar{\mathcal{U}}_{i,2} \otimes \dots \otimes \Bar{\mathcal{U}}_{i,l_i}$$
The noise channel representing the circuit $C$ under a noise model $\mathcal{E}$ can be represented as
$$\Bar{\mathcal{U}}_C = \Bar{\mathcal{U}}_1 \circ \Bar{\mathcal{U}}_2 \circ \dots \circ \Bar{\mathcal{U}}_L$$

The fidelity of the circuit $C$ under the noise model $\mathcal{E}$, denoted as $F_\mathcal{E}(C)$, is determined based on the fidelity of all pure states passing through the channels $\mathcal{U}_C$ and $\Bar{\mathcal{U}}_C$. Specifically, it is given by:

$$F_\mathcal{E}(C) = \int d\psi PF (U_C|\psi\rangle \langle \psi|U_C^\dagger, \Bar{\mathcal{U}}_C(|\psi\rangle \langle \psi|))$$

This notation, $F_\mathcal{E}(C)$, will be used throughout the paper to represent the fidelity of a given circuit.

\section{Problem definition}
\label{sec:problem_definition}



In this section, we present the problem of optimizing the routing stage in the transpilation process to maximize the fidelity of the resulting circuit. While keeping the layout and translation strategies fixed, we focus on developing a routing method specifically aimed at maximizing the circuit fidelity. To independently represent the outcome of the routing stage, we introduce the concept of \emph{gate sequence}. We also outline the constraints that define valid solutions within this stage. We then demonstrate the completeness of this concept by showing how a gate sequence transforms into a unique circuit—the final output of the transpilation process—and proving that any valid circuit can be derived from a gate sequence. This foundation is essential for clearly defining the problem and effectively analyzing our proposed framework later. Finally, we formally define the problem addressed in this paper.

\subsection{Concepts and Constraints in the Routing Stage}

In this section, we first introduce the essential concepts related to the routing stage, which are categorized into logical concepts from the quantum algorithm and hardware concepts from the quantum hardware. Next, we propose a representation for the routing stage solution, referred to as the gate sequence. Finally, we define the constraints that the routing stage solution must satisfy.

We start with logical concepts from the quantum algorithm. In a quantum algorithm, the set of logical qubits is defined as $\mathcal{Q}^L = \{q^L_1, \dots, q^L_n\}$. The set of all possible gate types used in the algorithm is defined as $\mathcal{T}$. For example, in the QAOA algorithm~\cite{Farhi2014}, we have $\mathcal{T} = \{R_{ZZ}, R_{X}\}$. A gate type $U \in \mathcal{T}$ applied on two qubits $q^L_u, q^L_v \in \mathcal{Q}^L$ is represented as a gate $s = (U, q^L_u, q^L_v)$. The quantum algorithm specifies a set of target gates to be executed on the quantum hardware, denoted as $S_{target}= \{ s = (U, q^L_u, q^L_v)|U \in \mathcal{T}, q^L_u, q^L_v \in \mathcal{Q}^L\}$. These gates must be executed in a specific order determined by their dependencies, which are captured in a directed dependency graph $G_D = (S_{target}, E_D)$. An illustration of a dependency graph is shown in Figure~\ref{fig:example:overview}.


Next, we introduce hardware concepts from the quantum hardware. The set of physical qubits in the hardware is denoted as $\mathcal{Q}^H = \{q^H_1, \dots, q^H_n\}$ and the physical connections between pairs of hardware qubits are represented by the set $E_H = \{(q^H_u, q^H_v)| q^H_u, q^H_v \in \mathcal{Q}^H\}$. Together, these define the hardware graph $G_H = (\mathcal{Q}^H, E_H)$. After the layout stage, a one-to-one mapping is established between logical qubits and physical qubits, represented by the function $\phi: \mathcal{Q}^L \rightarrow \mathcal{Q}^H$. This mapping is referred to as the \emph{assignment}. During the routing stage, the assignment can be modified by introducing SWAP gates. The set of all possible SWAP gates is denoted as $S_w = \{(SWAP, q^L_u, q^L_v)| q^L_u, q^L_v \in \mathcal{Q}^L\}$. Specifically, given the current assignment $\phi$, if there is a SWAP gate acting on two qubits $q^L_u$ and $q^L_v$, then the new assignment $\phi'$ is constructed as $\phi'(q^L_u) = \phi(q^L_v)$, $\phi'(q^L_v) = \phi(q^L_u)$, and $\phi'(q^L_t) = \phi(q^L_t)$ for all $t \neq u, v$.



To enable a formal definition of the problem in our study, we introduce a representation for the routing stage solution, referred to as the gate sequence. Given a gate limit $T$, the solution of the routing stage can be represented as a gate sequence $\mathbf{s} = [s_1, s_2, \dots, s_{T}]$, where each $s_i \in S_{target} \cup S_w$. Specifically, the gate sequence defines the order of gate execution on the hardware, where $s_i$ is not executed after $s_{i+j}$ for any $j>0$. We define the subsequence $\mathbf{s}^{(i)} = [s_1, s_2, \dots, s_i] \subseteq \mathbf{s}$ as the sequence of the first $i$ gates. Since the assignment may change during gate execution, we denote the assignment after executing the $i$-\emph{th} gate as $\phi_i$ with the initial assignment $\phi_0$ given as the output of the layout stage. 


A valid solution of the routing stage (i.e., valid gate sequence) must satisfy three constraints. The first is the \emph{inclusion constraint}, which requires that all target gates in $S_{target}$ must be included in the sequence $\mathbf{s}$. In other words, it must hold that $s \in \mathbf{s} \forall s \in S_{target}$. The second constraint is the \emph{dependency constraint}, which ensures that 
$\mathbf{s}$ respects the dependencies in the directed dependency graph $G_D$. Specifically, for any two gates $s_i, s_{i+j} \in \mathbf{s}$ with $j > 0$, there is no path from $s_{i+j}$ to $s_i$ in the dependency graph $G_D$. Lastly, the \emph{connectivity constraint} requires that for any gate $s_i = (U, q^L_u, q^L_v)$, the physical qubits corresponding to $q^L_u$ and $q^L_v$ in the assignment $\phi_{i-1}$ must be directly connected, i.e., $(\phi_{i-1}(q^L_u), \phi_{i-1}(q^L_v)) \in E_H$.

\subsection{Transformation between Gate sequences and Circuits}

\begin{figure*}[t]
  \centering
  \includegraphics[width=1\textwidth]{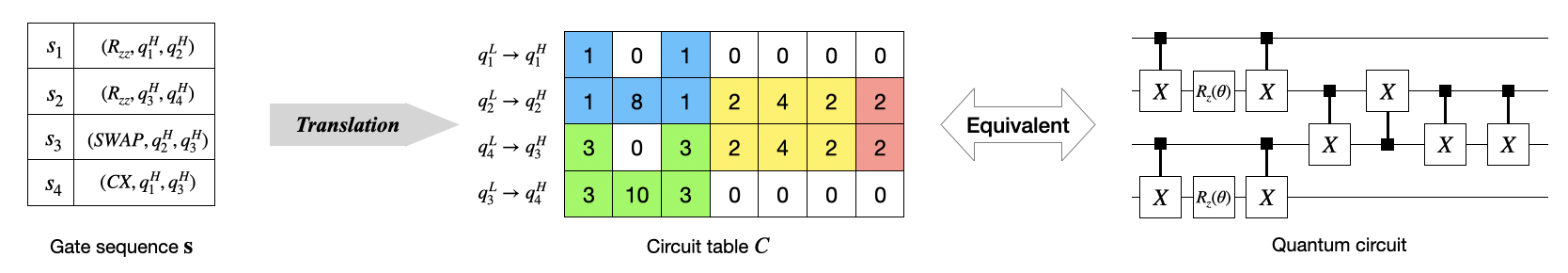}
  \label{fig:translation}
\caption{An example of how Algorithm~\ref{alg:translation} translates a gate sequence into a circuit table. In this process, gates in the sequence are decomposed into native gates and sequentially added to the table. Specifically, the $R_{zz}$ gate is decomposed into two $C_X$ gates and one $R_z$ gate, while the $SWAP$ gate is decomposed into three $C_X$ gates. The colored cells in the circuit table represent the identities of the native gates, corresponding to the gate of the same color in the gate sequence. In addition, each circuit table uniquely corresponds to a single circuit, and vice versa. Therefore, the two terms can be used interchangeably.}
\label{fig:example:translation}
\end{figure*}

We begin by introducing an alternative representation for the circuit, referred to as the \emph{circuit table}, which offers a more convenient format for explaining our proposed learning-based framework later. We then introduce an algorithm that translates a gate sequence into its corresponding circuit table. This algorithm is crucial in bridging the gate sequence to the resulting circuit, whose fidelity serves as the measure of the gate sequence's quality.


The circuit provides a visual representation of the sequence of operations (gates) applied to qubits over time. While this format is intuitive for visualization, it is not convenient for the mathematical computations. To address this, a circuit can be alternatively represented as a circuit table $C$ with $M$ columns and $N$ rows where $M$ is the depth limit and $N$ is the qubit limit (see Figure~\ref{fig:example:translation}). In this representation, the entry $C[i][j]$ specifies the identity of the native gate applied to the qubit i-\emph{th} at the layer j-\emph{th}. The row i-\emph{th}, denoted as $C[i][:]$, captures the sequence of native gates applied to the qubit i-\emph{th} throughout the circuit, while the column j-\emph{th}, denoted as $C[:][j]$, represents native gates executed in parallel in the layer j-\emph{th}. To ensure clarity and avoid ambiguity, native gates of the same type acting on different qubits or qubit pairs are assigned distinct identities. For example, CX gates acting on qubit pairs $(1,2)$ and $(3,4)$ are assigned two different identities. This distinction allows for a unique reconstruction of a quantum circuit from a circuit table. Thus, in this work, we use the notion of circuit and circuit table interchangeably and denote both notions as $C$.

Next, we introduce the algorithm for translating a gate sequence into a circuit table. Given a valid sequence $\mathbf{s} = \{s_1, \dots, s_{T_s}\}$ with $s_i \in S_{target} \cup S_w$ and the initial assignment $\phi_0$, the objective is to place all gates in $\mathbf{s}$ into a circuit table $C$. The translation from $\mathbf{s}$ to $C$ is outlined in Algorithm~\ref{alg:translation}. Specifically, the algorithm takes the valid sequence $\mathbf{s} = \{s_1, \dots, s_m\}$, the initial assignment $\phi_0$, and the dependency graph $G_D$ as inputs. The output is the constructed circuit table $C$. The algorithm operates under the assumption that the strategies for the layout and translation stages are predetermined. Specifically, the initial assignment $\phi_0$ (from the layout stage) and the decomposition of gates in $\mathbf{s}$ (from the translation stage) are fixed.

The details of Algorithm~\ref{alg:translation} are described as follows. First, an empty table $C$, an array \texttt{qubit\_layer}, and an array \texttt{gate\_layer} are initialized with all entries set to $0$ (lines 1-3). For each qubit $i$, \texttt{qubit\_layer[i]} tracks the last layer where a gate was applied to qubit $i$. On the other hand, for each gate $s \in \mathbf{s}$, \texttt{gate\_layer} records the last layer where the gate $s$ was implemented. These arrays are updated as we place gates into $C$. Next, we iterate over the gates in $\mathbf{s}$ in sequential order. At step $i$, the triple $(U, q^L_u, q^L_v)$ is extracted from the gate $s_i$ (line 5). Note that $U$ may not be a native gate type, so $s_i$ needs to be decomposed into a set of native gates $\{\Bar{s}_1, \dots, \Bar{s}_{m_i}\}$ (line 6). These native gates will be placed in consecutive layers. Then, we must determine the appropriate layer, denoted as \texttt{chosen\_layer}, for placing the set of $m_i$ native gates. First, we extract the corresponding physical qubits $q^H_u$ and $q^H_v$ based on the assignment $\phi_{i-1}$ (lines 7-8). The \texttt{chosen\_layer} is assigned as the closest available layer for both qubits $q^H_u$ and $q^H_v$ (line 9). Additionally, \texttt{chosen\_layer} must follow the last layers of prerequisite gates of $s_i$ in the dependency graph $G_H$ (lines 10-12). If $s_i$ is a SWAP gate, \texttt{chosen\_layer} must be the latest available layer to ensure that the SWAP gate is executed after all previous gates (lines 13-14). Once the appropriate layer is found, the native gates are placed into the circuit table $C$ at the \texttt{chosen\_layer} (lines 15-17). The arrays \texttt{qubit\_layer}, \texttt{gate\_layer} and the assignment $\phi_i$ are also updated accordingly (lines 18-24). The return circuit $C$ (line 25) corresponding to the input $\mathbf{s}$ is unique. Thus, we denote the mapping from gate sequence to circuit based on Algorithm~\ref{alg:translation} as $\Psi$. In this work, we assume that the initial assignment $\phi_0$ is fixed. Thus, given a gate sequence $\mathbf{s}$ and a dependency graph $G_D$, the function $\Psi$ produces a unique circuit $C$. We can formally write $C = \Psi(\mathbf{s}, G_D)$. The quality of a gate sequence $s$ can therefore be directly assessed by evaluating the fidelity of this resulting circuit, which connects our notations as $F_{\mathcal{E}}(\Psi(\mathbf{s}, G_D))$.

\begin{algorithm}[ht]
\DontPrintSemicolon
  \KwInput{The gate sequence $\mathbf{s} = \{s_1, \dots, s_{T_s}\}$, the initial assignment $\phi_0$, the dependency graph $G_D$}
  \KwOutput{The circuit table $C$}
  
  Initialize an empty circuit table $C$ with $N$ rows and $M$ columns.

  Initialize an array $\texttt{qubit\_layer[i]} = 0 \forall i \in [1,N]$.

  Initialize an array $\texttt{gate\_layer[i]} = 0 \forall i \in [1,m]$.

  \For{$i:= 1 \to T_s$}{
    Extract $(U, q^L_u, q^L_v)$ from $s_i$.

    Decompose $s_i$ into the set of $m_i$ native gates $\{\Bar{s}_1,\dots, \Bar{s}_{m_i}\}$

    $q^H_u \leftarrow \phi_{i-1}(q^L_u)$

    $q^H_v \leftarrow \phi_{i-1}(q^L_v)$

    $\texttt{chosen\_layer} \leftarrow \max(\texttt{qubit\_layer}[q^H_u], \texttt{qubit\_layer}[q^H_v]) + 1$

    \For{$j:= 1 \to i-1$}{
        \If{ $\exists$ a path from $s_j$ to $s_i$ in $G_D$} {
            $\texttt{chosen\_layer} \leftarrow \max(\texttt{chosen\_layer}, \texttt{gate\_layer[j]} + 1)$
        }

        \If{ $U \equiv SWAP$} {
            $\texttt{chosen\_layer} \leftarrow \max(\texttt{chosen\_layer}, \texttt{gate\_layer[j]} + 1)$
        }
    }
    \For{$j:= 0 \to m_i-1$}{
        $C[\texttt{chosen\_layer}+j][q^H_u] \leftarrow \Bar{s}_j$

        $C[\texttt{chosen\_layer}+j][q^H_v] \leftarrow \Bar{s}_j$
    }
    
    $\texttt{gate\_layer[i]} \leftarrow \texttt{chosen\_layer}+m_i-1$

    $\texttt{qubit\_layer}[q^H_u] \leftarrow \texttt{chosen\_layer}+m_i-1$

    $\texttt{qubit\_layer}[q^H_v] \leftarrow \texttt{chosen\_layer}+m_i-1$

    $\phi_i \equiv \phi_{i-1}$
    
    \If{$U \equiv SWAP$} {
        $\phi_i(q_u^L) \leftarrow \phi_{i-1}((q_v^L))$
        
        $\phi_i(q_v^L) \leftarrow \phi_{i-1}((q_u^L))$
    }
  }
  \Return $C$
\caption{Translation}
\label{alg:translation}
\end{algorithm}

\subsection{Fidelity Maximization in Routing Stage Problem}

We are now ready to formally define the Fidelity Maximization in Routing Stage (FMRS) problem as follows: 

\begin{definition}
    (FMRS problem) Given a set of target gates $S_{target}$, a set of SWAP gates $S_w$, and a dependency graph $G_D$, a hardware graph $G_H$, a gate limit $T$ and a noise model $\mathcal{E}$, the fidelity maximization problem aims to construct a valid sequence $\mathbf{s} = \{s_1, \dots, s_{T_s}\}$ where $s_i \in S_{target} \cup S_w, \forall s_i\in \mathbf{s}$ and $s \in \mathbf{s}, \forall s \in S_{target}$ such that the fidelity $F_{\mathcal{E}}(\Psi(\mathbf{s}, G_D))$ is maximized and $T_s \leq T$.
    \label{def:fidelity_maximization}
\end{definition}

As can be seen, FMRS aims to find a valid sequence $\mathbf{s}$ that satisfies the constraints of the routing stage, contains no more than $T$ gates and maximizes the fidelity of the translated circuit under a noise model $\mathcal{E}$. This problem definition establishes a clear goal and provides a foundation for analyzing routing techniques with the corresponding goal.

\section{Technical method}
\label{sec:proposedapproach}
This section presents our proposed framework, \NAMEA, as two key modules. First, we introduce the GP-based surrogate model which estimates the fidelity of given circuits. Then, we propose an 
RL model designed to optimize the gate sequence in order to maximize the resulting circuit fidelity, utilizing the fidelity estimation of the surrogate model as reward signals. 

\subsection{Surrogate Model for Estimating the Circuit Fidelity}

\begin{figure*}[t]
  \centering
  \includegraphics[width=1\textwidth]{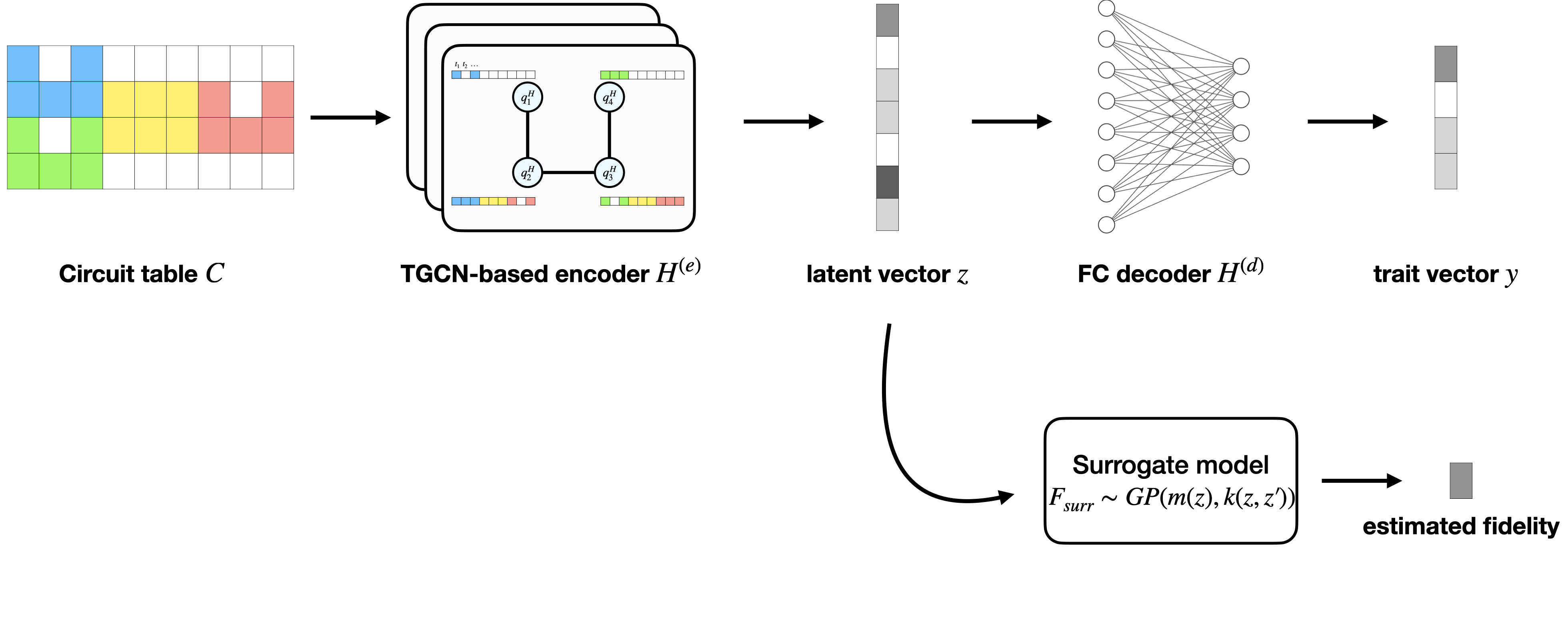}
  \caption{The flow for training the proposed surrogate model ....}
  \label{fig:surrogate_model}
\end{figure*}

In this part, we introduce an efficient technique to estimate the fidelity of quantum circuits by 1) embedding them into a continuous space, and 2) building a surrogate model based on Gaussian Process in this space. The estimated fidelity then is used to provide reward signals for training our proposed RL model.

\subsubsection{Circuit Embedding using TGCN-based Autoencoder}
To facilitate smooth and efficient learning of the relationship between the circuit and its corresponding fidelity, it is essential to embed a discrete circuit into a continuous latent space. The architecture of the TGCN-based autoencoder employed for embedding circuits into the latent space is illustrated in Figure~\ref{fig:surrogate_model}.

Recall that a circuit can be represented as a table $C$ with $M$ columns corresponding to the depth limit and $N$ rows corresponding to the qubit limit (i.e., $N = |\mathcal{Q}^L| = |\mathcal{Q}^H|$ and $M = T$). Each entry in $C$ represents a gate applied in a specific qubit in a specific layer. At each time step, all gates in the same layer are placed and executed simultaneously on the hardware before proceeding to the next layer, implying temporal dependency between the columns of $C$. In addition, the $N$ qubits have spatial dependency, represented by the hardware graph $G_H$. Therefore, to fully capture the dynamics of the system, it is necessary to adopt a learning structure that can account for both the temporal dependency represented by $C[:][1], \dots, C[:][M]$ and the spatial dependency presented by $G_H$.

By embedding circuits into a continuous latent space, our goal is to ensure that circuits with similar fidelity are mapped to nearby points in this space. This enables smoother interpolation between circuits, facilitating the application of Gaussian Process to construct a surrogate model for fidelity estimation. However, existing methods for calculating exact fidelity, such as direct fidelity estimation, are computationally expensive, leading to an insufficient number of data points for effective training. To address this problem, we identify a set of alternative fidelity-related traits that are strongly correlated with fidelity but significantly easier to compute. Circuits with similar traits are mapped to proximate points in the latent space. The list of traits can be denoted as a vector $\mathbf{y} = [y_1, \dots, y_{d_y}]$. In this work, the list of traits include the circuit depth, and the gate count for each type of gate in total. Additional hardware-specific noise features, such as qubit-specific $T_1$ and $T_2$ times, gate error rates, readout errors, and calibration drift, can be readily incorporated into our framework. To include such features, we can simply augment the list of fidelity-related traits with these hardware-specific parameters and retrain the model. Our framework is flexible and agnostic to the feature set, so it can seamlessly integrate richer noise profiles as long as they are available.

The goal of the autoencoder is to learn two key functions: the encoder $H^{(e)}$ and the decoder $H^{(d)}$. The encoder $H^{(e)}$ maps the sequence of $M$ feature vectors $C[:][t]$ with $t \in [1,M]$ and the hardware graph $G_H$ into an embedding vector $z \in \mathbb{R}^{d_z}$. The decoder $H^{(d)}$ aims to reconstruct the fidelity-related vector $\mathbf{y}$ from the embedding vector $z$. The overall process can be summarized as:

$$(C[:][1], \dots, C[:][M];G_H) \xrightarrow{\text{$H^{(e)}(.)$}} z \xrightarrow{\text{$H^{(d)}(.)$}} \mathbf{y}$$

For learning the encoder $H^{(e)}$, we use the TGCN model~\cite{Yan2018,Zhao2020} which is efficient in capturing both spatial and temporal dependency of the input. The TGCN based model is a combination of the GCN model $H^{(gcn)}$, and recurrent gated units. Specifically, 2-layer GCN model is used to capture the spatial dependency of qubit connectivity. Given the adjacency matrix of $G_H$ as $A_H$, we have the adjacency matrix with added self-connections $\hat{A}_H = A_H+I$, the diagonal matrix of degrees $\hat{D}$ where $\hat{D}_{ii} = \sum_{j} \hat{X}_{H ij}$ and $\tilde{A}_H = \hat{D}^{\frac{-1}{2}}\hat{A}_H\hat{D}^{\frac{-1}{2}}$. Given the feature matrix at step $t$ as $C[:][t]$, the two layer GCN model can be given as:
\begin{align}
    H^{(gcn)}_{\varphi_{gcn}} (C[:][t], G_H) = \sigma(\tilde{A}_HReLU(\tilde{A}_HC[:][t]W_{gcn,0})W_{gcn,1})
    \label{eq:gcn}
\end{align}

where $\varphi_{gcn} = [W_{gcn,0}, W_{gcn,1}]$ are the trainable weights of two layers, and $\sigma$ and $ReLU$ are activation functions. Based on the GCN model $H^{(gcn)}$, the forwarding computation of TGCN based model can be illustrated as follows:

$$e_{t,1} = \sigma (W_{t,1}[H^{(gcn)}_{\varphi_{gcn}} (C[:][t], G_H), h_{t-1}]+b_{t,1})$$

$$e_{t,2} = \sigma (W_{t,2}[H^{(gcn)}_{\varphi_{gcn}} (C[:][t], G_H), h_{t-1}]+b_{t,2})$$

$$e_{t,3} = \tanh (W_{t,3}[H^{(gcn)}_{\varphi_{gcn}} (C[:][t], G_H), r_{t,2}\times h_{t-1}]+b_{t,3})$$

$$h_{t} = e_{t,1} \times h_{t-1}+(1-e_{t,1}) \times e_{t,3}$$

$$\forall t = 1,\dots, M$$
where $h_t$ is the output at time $t$, $e_{t,1}$, $e_{t,1}$ and $e_{t,1}$ are recurrent gates at time $t$, and $W_{t,i}$, and $b_{t,i}$ with $t \in [1,M], i\in [1,2,3]$ are corresponding trainable weights and bias. The output of the $TGCN$ model, denoted as $H^{(tgcn)}_{\varphi_{tgcn}}$, corresponds to the final layer's output. This can be expressed as:
\begin{align}
    H^{(tgcn)}_{\varphi_{tgcn}}(C, G_H) = h_M
    \label{eq:tgcn}
\end{align}
where $\varphi_{tgcn} = \varphi_{gcn} \cup [W_{t,i}$, $b_{t,i}]_{t \in [1,M], i\in [1,2,3]}$ are trainable weights of the TGCN model. To facilitate further computations, we flatten the output of $H^{(tgcn)}_{\varphi_{tgcn}}$. Then, the latent vector is obtained as follows:

$$z = H^{(e)}_{\varphi_e}(C, G_H) = ReLU(W_eH^{(tgcn)}_{\varphi_{tgcn}} (C, G_H) + b_e)$$
where $\varphi_e = \varphi_{tgcn} \cup [W_e, b_e]$ are trainable weights. On the other hand, the decoder is in form of a 2-layer fully connected network. Specifically, we have:

$$H^{(d)}_{\varphi_d} (z) = W_{d,1}ReLU(W_{d,2}z + b_{d,2})+b_{d,1} = \Bar{\mathbf{y}}$$

where $\varphi_d = [W_{d,1}, b_{d,1}, W_{d,2}, b_{d,2}]$ are trainable weights of the 2-layer fully connected network, and $\Bar{\mathbf{y}}$ is the estimated output vector.

Given a training set of $n_{embed}$ instances $\{C_i, \mathbf{y}_i\}$ with a same hardware graph $G_H$, the loss function for the autoencoder model is:
$$\mathcal{L}_{embed}(\varphi) = \sum_{i=1}^{n_{embed}}||H^{(d)}_{\varphi_d} (z)(H^{(e)}_{\varphi_e}(C_i, G_H)) - \mathbf{y}_i||^2_2$$
where $\varphi = \varphi_e \cup \varphi_d$. Then, given the learning rate as $\alpha_{embed}$, the trainable weights are updated by the stochastic gradient descent as follows:
$$\varphi \leftarrow \varphi - \alpha_{embed} \nabla_\varphi \mathcal{L}_{embed}(\varphi)$$

\subsubsection{Gaussian Process for the Surrogate Model}
After embedding circuits into the latent space $\mathbb{R}^{d_z}$, the task of estimating fidelity becomes a regression problem within this continuous space. Gaussian Process (GP) regression is particularly well-suited for this task because it can achieve high predictive accuracy without requiring a large number of exact fidelity measurements, which are computationally expensive. In addition, GP provides not only point estimates but also a full posterior distribution over the predictions. This allows us to gain insights into the "shape" of the fidelity estimation function under different noise models.

The training process for the GP model is detailed as follows. Starting with a set of sampled circuits $\{C_i\}_{i=1}^{n_s}$, we compute their corresponding latent vectors $\mathbf{z} = \{z_1, \dots, z_{n_s}\}$ where $z_i = H^{(e)}_{\varphi_e} (C_i, G_H)$. Under a noise model $\mathcal{E}$, the fidelity of each circuit $C_i$ can be calculated as $u_i = F_\mathcal{E}(C_i) + \epsilon$ (i.e., we utilize the direct fidelity estimation in~\cite{Flammia_2011} to calculate the exact fidelity) with $\epsilon \in \mathcal{N}(0, \gamma^2)$. We train the GP model $\tilde{F}_{\mathcal{E}}$ using the training set $\{(z_i, u_i)\}_{i=1}^{n_s}$ to predict the fidelity of unseen circuits. The GP model $\tilde{F}_{\mathcal{E}}$ is defined as 
$$\tilde{F}_{\mathcal{E}}(z) \sim GP(m(z), k(z, z'))$$
where $m(z)$ is the mean function (assumed to be 0), and $k(z, z')$ is the kernel function that captures the covariance between points in the latent space. 

GP assumes that the outputs are drawn from a multivariate Gaussian distribution. With the training outputs as $\mathbf{u} = \{u_1, \dots, u_{n_s}\}$, we have the prior distribution on the outputs as:
$$\mathbf{u}|\mathbf{z} \sim \mathcal{N}(\mathbf{u}|0, \mathbf{K} + \gamma^2 \mathbf{I})$$
where $\mathbf{K}_{ij} = k(z_i, z_j)$.

Now, we describe how the GP model makes the estimation on an unseen input. Given a new circuit with latent representation $z_*$, GP aims to find the posterior distribution of the corresponding fidelity $u_*$. The joint distribution of $\mathbf{u}$ and $u_*$ can be described as:
$$\begin{bmatrix}
\mathbf{u} \\
u_*
\end{bmatrix} = \mathcal{N}(0, \begin{bmatrix}
 \mathbf{K} + \gamma^2 \mathbf{I} &  \mathbf{K}_*\\
 \mathbf{K}_*^T &  \mathbf{K}_{**}
\end{bmatrix})$$

where $\mathbf{K}_{* i} = k(z_i, z_*)$ and $\mathbf{K}_{**} = k(z_*, z_*)$. From this, we derive the posterior distribution for $u_*$ as follows:

$$u_*|\mathbf{z}, \mathbf{u}, z_* \sim \mathcal{N} (u_*|m_*, \sigma_*^2)$$
where $m_* = \mathbf{K}_*^T[\mathbf{K}+ \gamma^2 \mathbf{I}]^{-1} \mathbf{u}$ represents the mean of the predicted fidelity and $\sigma_*^2 = \mathbf{K}_{**} -\mathbf{K}_*^T[\mathbf{K}+ \gamma^2 \mathbf{I}]^{-1} \mathbf{K}_*$ quantifies the uncertainty associated with the prediction.

The GP surrogate model is designed to be robust against errors in the limited estimates used for training. Robustness comes from its foundational modeling assumptions. When training the surrogate model, the fidelity estimates, $u_i$, are not treated as perfect ground truth. Instead, they are modeled as the sum of the true fidelity, $F_E(C_i)$, and a noise term, $\epsilon$, which is assumed to follow a normal distribution, $\mathcal{N}(0, \gamma^2)$. By explicitly incorporating this error term into the regression, the GP model learns to account for statistical noise present in the training data. Thus, this allows the model to avoid overfitting to inaccurate fidelity estimates and to produce more reliable predictions with associated confidence intervals.

The accuracy of GP mainly depends on the selection of the kernel function $k(z, z')$ which captures the underlying patterns in the data. An appropriately chosen kernel not only enhances predictive performance but also provides valuable insights into the "shape" of the fidelity estimation function under different quantum noise models. A comprehensive analysis of kernel selection and its impact is presented in the experimental section.

\subsubsection{Training Data Selection}
The high computational cost of the direct fidelity estimation in generating exact fidelity samples imposes a strict limit on the number of training data points. Thus, it is essential to carefully select a set of training points that are highly informative in training the model. Here, we propose an algorithm for selecting $n_s$ training points (i.e., latent vectors) that maximize the informativeness of the training set.

Based on the work~\cite{Srinivas2010}, the informativeness of a input set $\mathbf{z} = \{z_1, \dots, z_{n_s}\}$ in training the GP model can be quantified as:
$$Inf(\mathbf{z}) = \frac{1}{2} log |I + \gamma^{-2}\mathbf{K}|$$ where $\mathbf{K}_{ij} = k(z_i, z_j)$. In this work, given the family of circuit $\mathcal{C}$ for which we aim to estimate fidelity, we aim to select a set of $n_s$ latent vectors from the set $\mathcal{Z} = \{H_{\phi_e}^{e} (C, G_H)|C \in \mathcal{C}\}$ such that the informativeness of the selected set is maximized. The selection process is outlined in Algorithm~\ref{alg:selection}.

In details, we begin by initializing the set $\mathbf{z}$ as empty. Let $n$ denote the size of the latent vector set $\mathcal{Z}$ (i.e., $n = |\mathcal{Z}|$), which can grow exponentially with the number of qubits and the circuit depth for the family $\mathcal{C}$. Given that it is computationally infeasible to evaluate all $n$ latent vectors in polynomial time, we randomly sample a subset $\mathcal{Z}_s \subseteq \mathcal{Z}$, with the sampling rate controlled by $\epsilon_s$. From this sampled subset $\mathcal{Z}_s$, we select the candidate that provides the highest information gain and append it to $\mathbf{z}$.

The function $Inf$ is proved as a monotone submodular function~\cite{Srinivas2010}. Thus, we can bound the informativeness of solutions returned by Algorithm~\ref{alg:selection} in Theorem~\ref{theorem:submodular}.

\begin{theorem}
Algorithm~\ref{alg:selection} achieves an approximation ratio of $(1-1/e-\epsilon_s)$
in expectation of the optimal informativeness.
\label{theorem:submodular}
\end{theorem}

\begin{proof}
    The set $\mathcal{Z}_s$ contains $s =\frac{n}{n_s} log(\frac{1}{\epsilon_s})$ elements (with repetitions) from $\mathcal{Z} \setminus \mathbf{z}$. Given the optimal set with optimal informativeness as $\mathbf{z^*}$, we can estimate the probability that $\mathcal{Z}_s \cap (\mathbf{z^*} \setminus \mathbf{z}) = \emptyset$ as follows:
    \begin{align}
        Pr[\mathcal{Z}_s \cap (\mathbf{z^*} \setminus \mathbf{z}) = \emptyset] &= (1-\frac{|\mathbf{z^*} \setminus \mathbf{z}|}{|\mathcal{Z} \setminus \mathbf{z}|})^s \nonumber \\
         &\leq e^{-s\frac{|\mathbf{z^*} \setminus \mathbf{z}|}{|\mathcal{Z} \setminus \mathbf{z}|}} \nonumber \\
         &\leq e^{-s\frac{|\mathbf{z^*} \setminus \mathbf{z}|}{n}} \label{eq:theorem:empty}
    \end{align}

We observe that $1 - e^{-\frac{s}{n}x}$ is a concave function of $x$ with $x = |\mathbf{z^*} \setminus \mathbf{z}| \in [0,n_s]$, so we have $1 - e^{-\frac{s}{n}x} \geq (1-\frac{x}{n_s})(1- e^{-\frac{s}{n}0}) + \frac{x}{n_s}(1- e^{-\frac{s}{n}n_s})$. Combining with the result in (\ref{eq:theorem:empty}), we can imply that:

\begin{align}
        Pr[\mathcal{Z}_s \cap (\mathbf{z^*} \setminus \mathbf{z}) \neq \emptyset] &\geq 1- e^{-s\frac{|\mathbf{z^*} \setminus \mathbf{z}|}{n}} \nonumber \\
        &\geq (1-e^{-s\frac{n_s}{n}})\frac{|\mathbf{z^*} \setminus \mathbf{z}|}{n_s} \nonumber \\
        &= (1-\epsilon_s) \frac{|\mathbf{z^*} \setminus \mathbf{z}|}{n_s} \label{eq:theorem:not_empty}
\end{align}

Algorithm~\ref{alg:selection} picks an element $z \in \mathcal{Z}_s$ that maximizes the information gain $\Delta(z|\mathbf{z}) = Inf(\mathbf{z} \cup \{z\}) - Inf(\mathbf{z})$. $\mathcal{Z}_s$ has equal chance to contain each element of $\mathbf{z^*} \setminus \mathbf{z}$. Thus, the probability of randomly selecting an element in $\mathcal{Z}_s \cap (\mathbf{z^*} \setminus \mathbf{z})$ is equal to that of randomly selecting an element in $\mathbf{z^*} \setminus \mathbf{z}$. As a result, we have:
\begin{align}
    \mathbb{E}[\Delta(z|\mathbf{z})] &\geq Pr[\mathcal{Z}_s \cap (\mathbf{z^*} \setminus \mathbf{z}) \neq \emptyset] \times \frac{1}{|\mathbf{z^*} \setminus \mathbf{z}|} \sum_{z \in \mathbf{z^*} \setminus \mathbf{z}} [\Delta(z|\mathbf{z})] \nonumber \\
    &\geq \frac{1-\epsilon_s}{n_s}\sum_{z \in \mathbf{z^*} \setminus \mathbf{z}} \Delta(z|\mathbf{z}) \text{(Based on (\ref{eq:theorem:not_empty}))}\label{eq:theorem:margin_gain}
\end{align}

Let $\mathbf{z}_i = \{z_1, \dots, z_i\}$ be the solution of Algorithm~\ref{alg:selection} after $i$ steps. From (\ref{eq:theorem:margin_gain}), we have:
\begin{align*}
    \mathbb{E}[Inf(\mathbf{z}_{i+1}) - Inf(\mathbf{z}_i)] &= \mathbb{E}[\Delta(z|\mathbf{z}_i)|\mathbf{z}_i]\\ &\geq \frac{1-\epsilon_s}{n_s}\sum_{z \in \mathbf{z^*} \setminus \mathbf{z}_i} \Delta(z|\mathbf{z}_i)\\
    &\geq \frac{1-\epsilon_s}{n_s} \Delta(\mathbf{z^*}|\mathbf{z}_i) \\
    &\geq \frac{1-\epsilon_s}{n_s} [ Inf(\mathbf{z^*})-Inf(\mathbf{z}_i)] \text{(Due to the submodularity of \emph{Inf})}
\end{align*}

By induction, we have:

\begin{align*}
    \mathbb{E}[Inf(\mathbf{z}_{n_s})] \geq (1-(1-\frac{1-\epsilon_s}{n_s})^{n_s})Inf(\mathbf{z^*}) \geq (1-\frac{1}{e} - \epsilon_s)Inf(\mathbf{z^*})
\end{align*}

\end{proof}

\begin{algorithm}[ht]
\DontPrintSemicolon
  \KwInput{The informativeness function $Inf$, the training size $n_s$, the sampling rate $\epsilon_s$}
  \KwOutput{The training set $\mathbf{z}$}
  
  $\mathbf{z} \leftarrow \emptyset$

  $n \leftarrow |\mathcal{Z}|$

  \For{$i:= 1 \to n_s$}{
    $\mathcal{Z}_{s}$ $\leftarrow$ a subset obtained by sampling $\frac{n}{n_s} log(\frac{1}{\epsilon_s})$ random elements from $\mathcal{Z} \setminus \mathbf{z}$.

    $z^* \gets \arg\max_{z \in \mathcal{Z}_{s}} Inf(\mathbf{z} \cup \{z\}) - Inf(\mathbf{z})$

    $\mathbf{z} \leftarrow \mathbf{z} + \{z^*\}$
  }

  \Return $\mathbf{z}$
\caption{Selection of the training set}
\label{alg:selection}
\end{algorithm}

\subsection{Reinforcement Learning model for the FMRS problem}

In this section, we introduce an RL module designed to address the FMRS problem. We begin by formulating the FMRS problem as a Markov Decision Process (MDP) and detailing its components and workflow. Next, we describe the actor-critic algorithm employed to solve the MDP of the FMRS problem.

\subsubsection{Markov Decision Framework for the FMRS problem}

To develop the reinforcement learning (RL) module for solving the FMRS problem, we first model it as a Markov Decision Process (MDP):

$$\mathcal{M}_{FMRS} = (\mathbf{O}, \mathbf{A},  \pi, r)$$

\begin{itemize}
    \item $\mathbf{O}$: The state $o_t \in \mathbf{O}$ includes the set of target gates $S_{target}$, the dependency graph $G_D$, the hardware graph $G_H$, and the gate sequence constructed up to the $t$\emph{th} step $\mathbf{s}^{(t)} = (s_1, \dots, s_{t})$. Specifically, we have $o_t = [S_{target}, G_D, G_H, \mathbf{s}^{(t)}]$.
    \item $\mathbf{A}$: The action is defined in form of $a_t = (g, q^H_u, q^H_v)$ with $g \in \mathcal{T}$ and $(q_u^H, q_v^H) \in E_H$. 
    \item $\pi$: The RL’s policy is for selecting actions.
    \item $r$: The immediate reward at each time step evaluates the quality of the selected action.
\end{itemize}


In order to describe the workflow of the RL module, we define a round of finding the optimal solution (i.e., gate sequence) for an input consisting of a target set $S_{target}$, the dependency graph $G_D$ and the hardware graph $G_H$ as an episode. Given a training set of inputs, the RL module executes multiple episodes to explore solutions of those inputs. Within each episode, the RL module constructs the solution incrementally through a series of steps. 
Specifically, at each step $t$ of an episode, the RL module transitions from the previous state $o_{t-1}$ to a new state $o_{t}$ based on the following process. First, the action $a_t = (g, q^H_u, q^H_v)$ is selected based on the policy $\pi$. The action is then converted to the gate $s_{t} = (g, \phi_{t-1}^{-1}(q^H_u), \phi_{t-1}^{-1}(q^H_v))$ where the assignment $\phi_0$ is initialized at the start of each episode and is updated at each step. The selected gate $s_{t}$ must satisfy three constraints including inclusion, dependency and connectivity. For the inclusion constraint, $s_t$ belong to $S_{target} \cup S_w$. For the dependency constraint, given a gate $s' \in S_{target}$, if there exists a path from $s'$ to $s_{t}$ in the dependency graph $G_D$, then $s' \in \mathbf{s}^{(t-1)}$. For the connectivity constraint, there is an edge between $q^H_u$ and $q^H_v$ in the hardware graph $G_H$. We notice that at least one gate in $S_w$ always satisfies all three constraints, ensuring that the episode can continue without being interrupted due to the absence of a valid action. Finally, the new state is updated as $o_t = (S_{target}, G_D, G_H, \mathbf{s}^{(t)} =\mathbf{s}^{(t-1)} \cup \{s_t\})$. The assignment $\phi^{(t)}$ is also updated based on $\phi^{(t-1)}$ and $s_t$ (see Section 4.1 for details on updating the assignment). The reward $r_{t}$ in step $t$ is calculated as: 
\[
r_t = 
\begin{cases} 
-1 & \text{if } s_t \in S_w \\
+1 & \text{if } s_t \in S_{target} \setminus \mathbf{s}^{(t-1)} \\
-\beta_1 |S_{target} \setminus \mathbf{s}^{(T)}| & \text{if } t > T \\
\beta_2 \tilde{F}_{\mathcal{E}}(H^{(e)}_{\varphi_e}(\Psi(\mathbf{s}^{(t)}, G_D), G_H)) & \text{if } |S_{target}| = 0 \\
\end{cases}
\]

Let us explain the reward function in more details. First, we aim to minimize redundant SWAP gates which may lead to unnecessary noise while promoting the selection of target gates. Therefore, we assign a negative reward for any SWAP gate $s_t \in S_w$ and a positive reward for selecting a target gate $s_t \in S_{target}$. In addition, we implement delayed rewards when reaching the terminal state of the episode. Specifically, an episode terminates when one of the following two conditions is met. The first termination condition is $t > T$. The occurrence of this condition indicates that the episode has resulted in an infeasible solution because the length of the gate sequence exceeds the allowed gate limit, i.e., $|\mathbf{s}^{(t)}| > T$.
If the solution is infeasible, a negative reward proportional to the number of remaining target gates is imposed. Conversely, if the solution is feasible, a positive reward based on the estimated fidelity of the final gate sequence is assigned. This fidelity is computed by translating the sequence into a circuit using Algorithm~\ref{alg:translation} (i.e., $\Psi$ function), embedding it into a latent space using the encoder $H^{(e)}_{\varphi_e}$, and estimating it with the GP-based fidelity estimator $\tilde{F}_{\mathcal{E}}$. The constants $\beta_1$ and $\beta_2$ are used to scale delayed rewards to an appropriate range. This reward structure encourages the pursuit of high-fidelity feasible solutions.

Information in an episode including states, actions, and rewards is stored in a replay buffer $\mathcal{B}$. After a predefined number of episodes, the policy $\pi$ is updated based on information stored in $\mathcal{B}$. Then, the updated policy is used to generate new episodes. This iterative process enables the RL module to learn an optimal policy for solving the FMRS problem effectively. In the subsequent section, we describe the representation of the policy and the Actor-Critic algorithm used to optimize the policy.

\subsubsection{Actor-Critic Algorithm for Solving the MDP}
To solve the MDP of the FMRS problem, we utilize the actor-critic deep reinforcement learning framework. In this approach, the policy is represented as two learning networks: actor and critic. The actor network outputs action probabilities, guiding decision-making. On the other hand, the critic network estimates state values, providing feedback to enhance learning.

\noindent{\bf Actor network.} In the MDP for the FMRS problem, the state $o_t = [S_{target}, G_D, G_H, \mathbf{s}^{(t)}]$ contains a substantial amount of information. We divide the state information into two types: physical and logical information. The physical information includes the hardware connectivity $G_H$, and the circuit $C = \Psi(\mathbf{s}^{(t)}, G_D)$.  To process this information, we apply a TGCN model following the formulation in Equation~\ref{eq:tgcn}. Specifically, the physical information is represented as:

\begin{align}
    \mathbf{X}_{physic} = H^{(tgcn)}_{\varphi_{physic}}(C, G_H)
\end{align}

where $\varphi_{physic}$ denotes the trainable parameters of the model.

Logical information includes the dependency graph $G_D$, the set of target gates $S_{target}$ and the partial solution $\mathbf{s}^{(t)}$. We can encode $S_{target}$ and $\mathbf{s}^{(t)}$ as a feature vector of $G_D$, denoted as $f_D \in \{0,1\}^{|S_{target}|}$ such that $f_D[s] = 1 \forall s\in S_{target} \setminus \mathbf{s}^{(t)}$ and $f_D[s]  = 0$ otherwise. To capture the logical dependencies, we apply a GCN model as outlined in Equation~\ref{eq:gcn}:
\begin{align}
    \mathbf{X}_{logic} = H^{(gcn)}_{\varphi_{logic}}(f_D, G_D)
\end{align}
where $\varphi_{logic}$ denotes trainable weights in the model. Finally, the actor network is formed by passing the combination of the physical and logical information to a fully connected layer:

$$H^{(actor)}_{\varphi_{actor}}(o_t) = W_{actor,2}ReLU(W_{a,1}[\mathbf{X}_{physic}, \mathbf{X}_{logic}] + b_{actor,1}) + b_{actor,2} = \mathbf{X}_{actor}$$
where $\varphi_{actor} = \varphi_{physic} \cup \varphi_{logic} \cup [W_{actor,1}, b_{actor,1}, W_{actor,2}, b_{actor,2}]$. Here, we have $\mathbf{X}_{actor} \in \mathbb{R}^{(|\mathcal{T}|\|E_H|) \times 1}$, where $|\mathcal{T}||E_H|$ corresponds to the dimension of the action space. The policy $\pi$ can be calculated from the actor's outputs as follows:
\begin{align}
    \pi_{\varphi_{actor}} (a_i|o_t) = \frac{exp(\mathbf{X}_{actor}[i])}{\sum_j exp(\mathbf{X}_{actor}[j])} \forall i \in [1,|\mathcal{T}||E_H|]
\end{align}

\noindent{\bf Critic network.}
Given the state $o_t = [S_{target}, G_D, G_H, \mathbf{s}^{(t)}]$, the critic network aims to evaluate the value of $o_t$. The structure of the critic network is similar to the structure of the actor network. Specifically, we have:

$$\mathbf{X'}_{physic} = H^{(tgcn)}_{\varphi'_{physic}}(C, G_H)$$

$$\mathbf{X'}_{logic} = H^{(gcn)}_{\varphi'_{logic}}(f_D, G_D)$$

$$H^{(critic)}_{\varphi_{critic}}(o_t) = W_{critic,2}ReLU(W_{critic,1}[\mathbf{X}'_{physic}, \mathbf{X}'_{logic}] + b_{critic,1}) + b_{critic,2} = \mathbf{X}_{critic}$$

where $\varphi_{critic} = \varphi'_{physic} \cup \varphi'_{logic} \cup [W_{critic,1}, b_{critic,1}, W_{critic,2}, b_{critic,2}]$ are the trainable parameters of the critic model. Here, we have $\mathbf{X}_{critic} \in \mathbb{R}$ which is the estimated value for the state $o_t$.

\noindent{\bf Updating the actor and critic networks.}
In the RL framework, the agent interacts with the environment in discrete time intervals. The interaction experience $(o_t, a_t, r_t, o_{t+1})$ is stored in a replay buffer $\mathcal{B}$, which is used to refine the weights of the actor and critic network. Given the discount factor $\gamma_r$, the loss function of the actor-critic algorithm is:
\begin{align}
\mathcal{L}(\varphi_{actor}, \varphi_{critic}) = \hat{\mathbb{E}}_{(o_t, a_t, r_t, o_{t+1}) \sim \mathcal{B}}\left[\log \pi_{\varphi_{actor}}\left(a_t \mid o_t\right) (r_t + \gamma_r H^{(critic)}_{\varphi_{critic}}(o_{t+1})- H^{(critic)}_{\varphi_{critic}}(o_t))\right]
  \label{eq:actor-critic loss}
\end{align}

Given the learning rates for the actor and critic models as $\alpha_{actor}$ and $\alpha_{critic}$, the parameters $\varphi_{actor}$ of the actor model and $\varphi_{critic}$ of the critic model can be updated as follows:

\begin{align}
\varphi_{actor} \leftarrow \varphi_{actor} - \alpha_{actor} \nabla_{\varphi_{actor}} \mathcal{L}(\varphi_{actor}, \varphi_{critic})
 \label{eq:actor update}
\end{align}

\begin{align}
\varphi_{critic} \leftarrow \varphi_{critic} - \alpha_{critic} \nabla_{\varphi_{critic}} \mathcal{L}(\varphi_{actor}, \varphi_{critic})
 \label{eq:critic update}
\end{align}

\subsubsection{Integration Within the Transpilation Pipeline}

The FIDDLE framework acts as a specialized routing engine that operates between the layout and translation stages in the quantum transpilation pipeline. Specifically, it treats the initial mapping of logical to physical qubits (layout) and the rules for decomposing gates into hardware-native operations (translation) as fixed inputs. This approach allows FIDDLE to fully dedicate its reinforcement learning capabilities to solving the complex routing problem by exploring the vast search space of gate sequences and SWAP insertions to find a solution that maximizes fidelity.

By working with a fixed initial layout and constant translation rules, FIDDLE’s performance naturally depends on the quality of those inputs, grounding the optimization in hardware realities. A well-chosen layout gives a good starting point, while fixed translation ensures fidelity estimates accurately capture the noise characteristics of the final physical circuit. This modular approach lets FIDDLE perform deep, fidelity-aware routing optimization without the complexity of handling all transpilation stages at once.

\section{Experiments}
\label{sec:experiments}
In this section, we first present the experimental setup. Next, we evaluate the effectiveness of our proposed surrogate model in estimating fidelity, which is the initial component of our solution. Finally, we compare the performance of our proposed framework, \NAMEA, with three state-of-the-art transpilation frameworks: Qiskit~\cite{qiskit}, VIC~\cite{Alam2020} and GNN-RL~\mbox{\cite{Saravanan_2024}}.

\subsection{Experimental Setup}
\label{sec:exp:setup}

\noindent {\bf Datasets.}
In our experiment, we focus on quantum circuits related to two illustrative quantum algorithms: the Quantum Approximate Optimization Algorithm (QAOA)\cite{Farhi2014} and Quantum Machine Learning (QML)\cite{Biamonte2017}. These algorithms were selected for three reasons. First, they represent two major domains in quantum computing which are quantum optimization and quantum machine learning. Improving circuit reliability in these domains is crucial for demonstrating practical quantum advantages. Second, both algorithms require repeated execution of their circuits to achieve final results. This makes circuit reliability even more critical, because errors compounding over multiple runs can directly impacting overall performance.
Finally, these algorithms involve a diverse set of target gates, and many of target gates can commute to each other. This introduces significant complexity to the routing stage because the execution order can vary and must be optimized to maximize fidelity. In contrast, algorithms like the Quantum Fourier Transform (QFT) have a fixed set of gates with a strict execution order, leaving little room for alternative routing strategies. As a result, testing on QAOA and QML circuits can highlight the advantages of our proposed routing method by showing its ability to handle complex and dynamic scenarios.

Let us analyze the differences between QAOA and QML circuits. In the case of QAOA, the set of gate types is  $\mathcal{T} = \{R_{ZZ}, R_{X}\}$ The target set of gates in QAOA is derived from the logical graph corresponding to the optimization problem being addressed. In this context, $R_{ZZ}$ gates are fully commutative with one another, whereas $R_{X}$ gates are not commutative with $R_{ZZ}$ gates. As a result, the dependency graphs associated with target gates in QAOA are sparse. For the QML application, we focus on the entanglement phase, where the gate set is $\mathcal{T} = \{R_{ZZ}, C_{X}\}$. The target set of gates is generated using the \emph{TwoLocal} library provided by Qiskit~\cite{Morales_2023}. The commutation rules in QML circuits result in more complex dependency graphs. Specifically, two $C_{X}$ gates commute only if they act on distinct pairs of control and target qubits. By applying our proposed approach, \NAMEA, to these two applications, we evaluate its performance across varying commutation properties and dependency graph structures, providing insights into its adaptability and effectiveness.

In this experiment, we construct four datasets comprising QAOA and QML circuits for 5-qubit and 7-qubit systems. These datasets are denoted as QAOA\_5, QAOA\_7, QML\_5 and QML\_7. Each instance within these datasets includes a set of target gates and the corresponding dependency graph, as specified by algorithms. All experiments in this study are performed using the \texttt{Cirq} quantum circuit simulator. The limitation to 5- and 7-qubit systems arises from the high computational cost associated with computing process fidelity, which we use as ground-truth labels for training the surrogate model. Estimating process fidelity for circuits with more qubits by existing methods in the literature such as direct fidelity estimation is extremely expensive and may take months to years to generate a sufficient number of training samples from scratch. This limitation is tied to data generation, not the scalability of our method itself. Once training samples sufficiently provided, our surrogate model and reinforcement learning framework can scale efficiently. Specifically, given $n_s$ fidelity training samples, the complexity of training our surrogate model is $\mathcal{O}(n_s^3 + n_s^2 d_z)$ where $d_z$ is the dimension of the latent space. In addition, the complexity of training our RL framework is $\mathcal{O}(B \times[M \times(|\mathcal{E}_H|+N)+|\mathcal{E}_D|])$ where $B$ is the batch size, $M$ is the circuit maximum depth, $N$ is the number of qubits, $|\mathcal{E}_H|$ is the number of hardware graph edges and $|\mathcal{E}_D|$ is the number of edges in the dependency graph.

\noindent {\bf Noise models.} 
In this work, we evaluate the performance of our proposed technique, \NAMEA, across various noisy environments defined by different noise models. We focus on five noise models: depolarizing, bit-flip, phase-flip, mix, and real. Specifically, for the depolarizing, bit-flip, and phase-flip models, we apply a noise rate of $0.008$ to two qubits, while a lower noise rate of $0.002$ is applied to all other qubits. The mix noise model applies all three types of noise simultaneously. Lastly, the real noise model replicates the actual noise conditions of the Rainbow quantum processors by Google~\cite{Stanisic2022}.

\noindent {\bf Training details.} 
For each dataset, we use $100$ instances for training the encoder and the surrogate model, denoted as $\mathcal{S}^{train}_{enc\_surr} = \{(S_{target}^{(i)}, G_D^{(i)})\}_{i=1}^{100}$. Each target set contains between $3$ and $15$ target gates. To train the encoder, we randomly generate $50$ circuits for each target set $S_{target}^{(i)} \in \mathcal{S}^{train}_{enc\_surr}$ using a random routing algorithm. Specifically, this algorithm constructs various gate sequences by randomly inserting target gates and SWAP gates. From these, $50$ valid sequences are selected and translated into circuits. For each circuit, we calculate its trait vector $\mathbf{y}$. This process results in a training set $\mathcal{K}^{train}_{enc} = \{C_i, \mathbf{y}_i\}_{i=1}^{5000}$ for training the encoder.

We then randomly select 300 circuits from $\mathcal{K}^{train}_{enc}$ and measure their fidelity using direct fidelity measurement~\cite{Flammia_2011}. We use $250$ samples for training and $50$ samples for evaluating the surrogate model. The trained encoder and surrogate model are then used to train the RL model. To ensure that the training data for the surrogate model and the RL model is independent, we generate an additional $100$ instances for training the RL model, denoted as $\mathcal{S}^{train}_{RL} = \{(S_{target}^{(i)}, G_D^{(i)})\}_{i=1}^{100}$. After training, the performance of the RL model is evaluated on $100$ unseen instances, denoted as $\mathcal{S}^{test}_{RL} = \{(S_{target}^{(i)}, G_D^{(i)})\}_{i=1}^{100}$. Each target set in $\mathcal{S}^{train}_{RL}$ and $\mathcal{S}^{test}_{RL}$ contains between $3$ and $15$ target gates.

In this experiment, we evaluate the performance of \NAMEA\ by comparing it with two transpilation frameworks: Qiskit~\cite{qiskit}, VIC~\cite{Alam2020} and GNN-RL~\mbox{\cite{Saravanan_2024}}. For Qiskit, we evaluate three routing algorithms—\emph{sabre}, \emph{lookahead}, and \emph{stochastic}—and report the highest fidelity achieved among them. On the other hand, VIC, which is specifically designed for the transpilation of QAOA circuits, is only included in the QAOA-related experiments. GNN-RL combines the GNN-based fidelity estimator from \mbox{\cite{Saravanan_2024}} with our proposed reinforcement learning (RL) method. We compare it with \mbox{\NAMEA}, which uses a GP-based fidelity estimator, to evaluate how the choice of estimator affects the overall performance of the learning framework. In this work, we do not consider the layout stage, so all benchmarks are conducted using a fixed initial assignment. We use the same hardware topology for the graph $G_H$ in every experiment, which is based on the grid architecture of the Rainbow processor.

\subsection{Evaluating the Performance of the Surrogate Model}
We assess the efficiency of our proposed method for estimating the fidelity in two steps. First, we analyze the latent space by showing the correlation between the distances of latent embeddings and their corresponding fidelity gaps. Next, we evaluate the performance of the surrogate model trained on these latent embeddings.

\subsubsection{Latent space analysis}
In our work, we embed circuits into a latent space and train a surrogate model using the embeddings in that space. The performance of the surrogate model directly depends on the quality of the latent space. An effective latent space is a low-dimensional representation that captures the essential features of the circuits necessary for predicting fidelity. In this experiment, we assess the quality of the latent space by analyzing the correlation between the similarity of embedding vectors (measured by Euclidean distance) and the similarity of their corresponding fidelities (measured by the L1 norm). A high correlation suggests that the optimal surrogate function in the latent space is likely smooth, which facilitates the construction of a more efficient and accurate GP model. Specifically, smoothness ensures that the GP can effectively capture the relationships in the latent space with fewer data points, improving interpolation and reducing computational complexity. 

Table~\ref{tab:correlation} presents the correlation results across different test cases and noise models. Overall, the correlation in all cases exceeds 
$0.5$, demonstrating a strong relationship between distances in the latent space and fidelity gaps. This result implies that neighboring points in the latent space tend to have similar fidelity. Therefore, our proposed framework is effective in learning a latent space that supports the efficient training of GP models.

In addition, we observe that the bit-flip model consistently yield the lowest correlations in most cases. Notably, for QML circuits on 5-qubit and 7-qubit systems, the bit-flip model results in correlations approximately $0.2$ and $0.3$ lower than the models with the highest correlations respectively. Let us delve deeper into the bit-flip noise model to understand why the similarity of latent embeddings shows a weaker correlation with the corresponding fidelity gap under this noise model. The low correlation can be attributed to the ability of swap gates to mitigate the effects of bit-flip errors in certain scenarios. Specifically, when bit-flip noise is activated on two qubits, its effect can be equivalent to applying a swap gate. As a result, the presence of multiple swap gates in a circuit can effectively cancel out the impact of bit-flip noise. Consequently, two circuits that are far apart in the latent space (e.g., due to significant differences in depth or gate count) can still exhibit similar fidelity, reducing the overall correlation.

\begin{table}[ht]
\begin{tabular}{|ll|l|l|l|l|l|}
\hline
 &  & Depolarized & Bit flip & Phase flip & Mix & Real \\ \hline
\multicolumn{1}{|l|}{\multirow{2}{*}{QAOA}} & Qubit = 5 & 0.9491 & 0.9237 & 0.9380 & 0.9383 & 0.9111 \\ \cline{2-7} 
\multicolumn{1}{|l|}{} & Qubit = 7 & 0.7014 & 0.6941 & 0.7107 & 0.6716 & 0.6973 \\ \hline
\multicolumn{1}{|l|}{\multirow{2}{*}{QML}} & Qubit = 5 & 0.7088 & 0.5057 & 0.6980 & 0.7007 & 0.7978 \\ \cline{2-7} 
\multicolumn{1}{|l|}{} & Qubit = 7 & 0.7945 & 0.5157 & 0.7881 & 0.7420 & 0.6928 \\ \hline
\end{tabular}
\caption{Correlation between distance in latent space and fidelity gap}
\label{tab:correlation}
\end{table}

\subsubsection{Surrogate model analysis}

In this section, we evaluate the performance of our proposed surrogate model and compare it with the GNN-based fidelity estimator~\cite{Saravanan_2024} to demonstrate the efficiency of our method in estimating the fidelity function.

Table~\ref{tab:surrogate:loss} reports the RMSE loss of our surrogate model using five different kernels: Exponential, RBF, Polynomial, MLP, and Rational Quadratic, across five noise models. The results indicate that the Exponential and MLP kernels are the most effective for estimating the fidelity function, achieving the best loss in most cases. Specifically, the Exponential kernel performs best for the 5-qubit system, while the MLP kernel excels in the 7-qubit system. In contrast, the RBF kernel, despite its popularity in GP applications, often performs poorly, yielding the highest loss in several cases. Notably, for the bit-flip noise model on the 5-qubit system, the RBF kernel produces a loss of $0.1193$, which is 16 times higher than the best-performing kernel.

We also observe a strong relationship between the correlation of latent embedding distances with fidelity gap and the surrogate model's performance. Noise models with high correlations tend to yield lower RMSE losses, highlighting the critical role of a well-structured latent space in improving the efficiency of GP-based models.

When compared with the RMSE loss of the GNN-based fidelity estimator, our GP model with the best-performing kernel consistently outperforms the GNN approach. Specifically, our surrogate model achieves RMSE losses that are $1.5$ to $2.5$ times lower than those of the GNN-based model. This demonstrates the superior capability of our generative approach in estimating fidelity functions, particularly when training data is limited, as compared to a discriminative model.

\begin{table}[]
\centering
\begin{tabular}{|c|c|c|c|c|c|c|c|}
\hline
Datasets & \multicolumn{2}{c}{Estimators} & \multicolumn{5}{|c|}{Noise models} \\ \cline{4-8} 
 & \multicolumn{1}{c}{}& & Depolarized & Bit flip & Phase flip & Mix & Real \\ \hline \hline
\multirow{6}{*}{QAOA\_5} & \multicolumn{2}{c|}{GNN-based~\cite{Saravanan_2024}} & 0.0122 & 0.0162 & 0.0117 & 0.0160 & 0.0235 \\ \cline{2-8}
& \multirow{5}{*}{GP-based}& {Exponential} & \textbf{0.0052} & \textbf{0.0075} & \textbf{0.0039} & \textbf{0.0093} & 0.0154 \\  
 & & RBF & 0.0077 & 0.1193 & 0.0145 & 0.0875 & 0.0150 \\  
 & & Poly & 0.0125 & 0.0182 & 0.0372 & 0.0099 & 0.0150 \\  
 & & MLP & 0.0057 & 0.0106 & 0.0090 & 0.0120 & \textbf{0.0148} \\  
 & & RatQuad & 0.0077 & 0.0107 & 0.0088 & 0.0148 & 0.0150 \\  \hline
 \multirow{6}{*}{QAOA\_7} & \multicolumn{2}{c|}{GNN-based~\cite{Saravanan_2024}} & 0.0255 & 0.0298 & 0.0245 & 0.0245 & 0.0323 \\ \cline{2-8}
 & \multirow{5}{*}{GP-based} & {Exponential} & 0.0091 & \textbf{0.0112} & 0.0049 & \textbf{0.0077} &  \textbf{0.0255}\\  
 & & RBF & 0.0112 & 0.0156 & 0.0080 & 0.0141 & 0.0508 \\  
 & & Poly & 0.0127 & 0.0231 & 0.0086 & 0.0157 &  0.0709\\  
 & & MLP & \textbf{0.0067} & 0.0115 & \textbf{0.0046} & 0.0082 & 0.0255 \\  
 & & RatQuad & 0.0069 & 0.0125 & 0.0049 & 0.0089 & 0.0255 \\  \hline
\multirow{6}{*}{QML\_5} & \multicolumn{2}{c|}{GNN-based~\cite{Saravanan_2024}} & 0.0297 & 0.0352 & 0.0295 & 0.0391 & 0.0311 \\ \cline{2-8}
& \multirow{5}{*}{GP-based} & Exponential & \textbf{0.0122} & \textbf{0.0260} & \textbf{0.0094} & \textbf{0.0189} & 0.0311 \\  
  & & RBF & 0.0270 & 0.0699 & 0.0200 & 0.0379 & 0.2184 \\  
 & & Poly & 0.0147 & 0.0344 & 0.0127 & 0.0393 &0.0325  \\  
 & & MLP & 0.0122 & 0.0260 & 0.0094 & 0.0189 & 0.0323 \\  
 & & RatQuad & 0.0124 & 0.0264 & 0.0104 & 0.0206 & \textbf{0.0306} \\  \hline
 \multirow{6}{*}{QML\_7} & \multicolumn{2}{c|}{GNN-based~\cite{Saravanan_2024}} & 0.0297 & 0.0464 & 0.0263 & 0.0287 & 0.0583 \\ \cline{2-8}
 & \multirow{5}{*}{GP-based} & Exponential & \textbf{0.0121} & 0.0408 & \textbf{0.0068} & 0.0163 & 0.0355 \\  
 & & RBF & 0.0183 & 0.0514 & 0.0110 & 0.0226 & 0.0643 \\ 
& & Poly & 0.0290 & 0.0896 & 0.0217& 0.0319 & 0.1194 \\  
 & & MLP & 0.0124 & \textbf{0.0407} & 0.0068 & \textbf{0.0158} & \textbf{0.0350} \\  
 & & RatQuad & 0.0125 & 0.0411 & 0.0074 & 0.0166 & 0.0356 \\  \hline
\end{tabular}
\caption{RMSE Loss comparison of the proposed GP-based surrogate model using different kernels against GNN-based estimator across five noise models.}
\label{tab:surrogate:loss}
\end{table}

\subsection{Evaluating the Performance of the End-to-End Framework}

In this section, we first present the training performance of our proposed framework, \NAMEA, to demonstrate its effectiveness in exploring the solution space. Next, we compare the performance of \NAMEA\ against state-of-the-art methods on unseen circuits to illustrate its efficiency in optimizing circuit fidelity in practical.

\subsubsection{Training performance}
We evaluate two key aspects of the training process. The first is the RL model's ability to identify feasible solutions (i.e., valid gate sequence). Ensuring that all constraints are satisfied is a fundamental prerequisite before addressing fidelity improvements. The second aspect is the effectiveness of the RL model in exploring solutions with high fidelity. This evaluation is essential for determining the capability of the RL approach to improve the circuit fidelity. To address these aspects, we analyze and present the training performance of our model in this experiment.

For the first aspect, we introduce the feasibility rate is the ratio of episodes resulting in a feasible solution to the total number of episodes within a fixed interval. 
In this work, we calculate the feasibility rate every interval of 200 episodes. Figure~\ref{fig:training:feasibility} illustrates the feasibility rate during the training process across four datasets and five noise models. The results show that our proposed RL method achieves the highest feasibility rates for the QAOA\_5, QAOA\_7, and QML\_5 datasets. Notably, even for the more complex QML\_7 dataset, \NAMEA\ attains a feasibility rate exceeding $0.9$. These results indicate that \NAMEA\ reliably finds feasible solutions with a very high probability.

For the second aspect, we define the fidelity improvement rate as the difference between the average fidelity achieved by \NAMEA\ and the average fidelity obtained using a random selection strategy on the training set. Specifically, for each training instance, the random selection strategy selects gates $g \in S_{target} \cup S_w$ randomly, representing the \NAMEA's policy before any training. Thus, the fidelity improvement rate quantifies how much \NAMEA\ enhances fidelity during training. 

Figure~\ref{fig:training:fidelity} illustrates the fidelity improvement rate across different datasets and noise models throughout the training process. A common upward trend in the fidelity improvement rate is observed across all cases. Specifically, after $25,000$ training episodes for 5-qubit datasets and $100,000$ episodes for 7-qubit datasets, fidelity improves by $1.5\%$ to $9\%$. This result demonstrates the effectiveness of our RL approach in optimizing fidelity.

The improvement rate of noise models, however, varies across different datasets. For the QAOA datasets, the improvement range is relatively consistent across all noise models, with values ranging from $2\%$ to $5\%$ for QAOA\_5 and $4\%$ to $7\%$ for QAOA\_7. Notably, the real noise model benefits the most from \NAMEA, showing the largest fidelity improvements in the QAOA datasets. Therefore, \NAMEA\ can be a valuable tool for improving fidelity in real-world quantum systems.

In contrast, for the QML datasets, the improvement rate exhibits greater variation across noise models, ranging from $1.5\%$ to $9\%$ for QML\_5 and $2\%$ to $8\%$ for QML\_7. Interestingly, the bit-flip noise model shows the least improvement compared to others. This indicates that bit-flip noise is harder to be improved in QML circuits. This phenomenon is likely due to the high loss of the surrogate model under this noise type. This finding highlights the importance of developing accurate surrogate models for optimizing fidelity effectively.

\begin{figure*}[t]
\begin{subfigure}{.24\textwidth}
  \centering
  \includegraphics[width=1\textwidth]{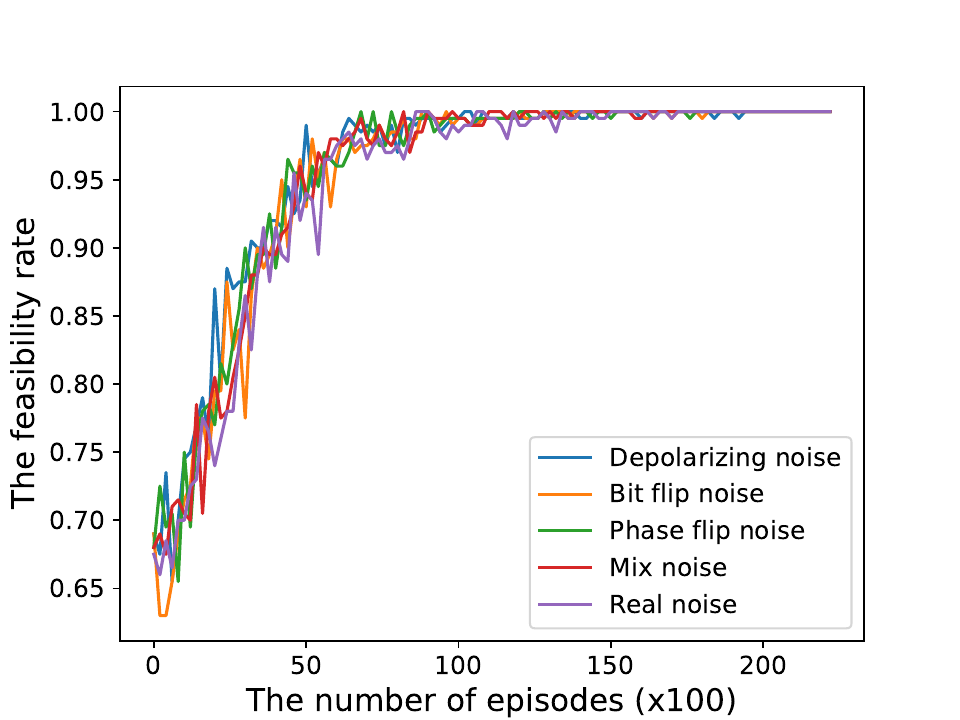}
  \caption{QAOA\_5}
  \label{}
\end{subfigure}
\hfill
\begin{subfigure}{.24\textwidth}
  \centering
  \includegraphics[width=1\textwidth]{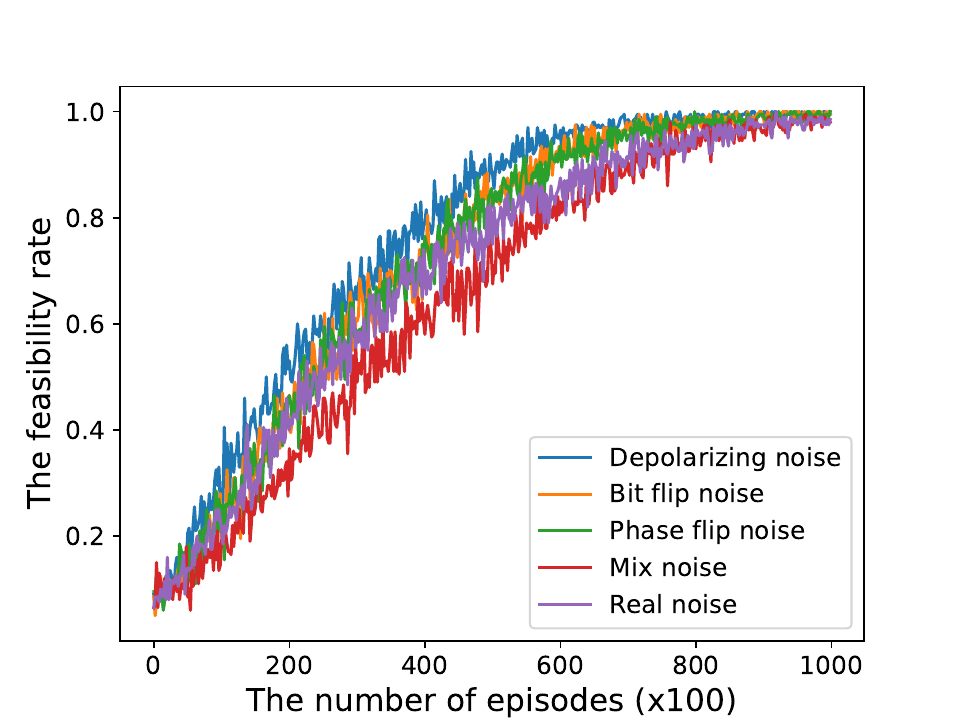}
  \caption{QAOA\_7}
  \label{}
\end{subfigure}
\hfill
\begin{subfigure}{.24\textwidth}
  \centering
  \includegraphics[width=1\textwidth]{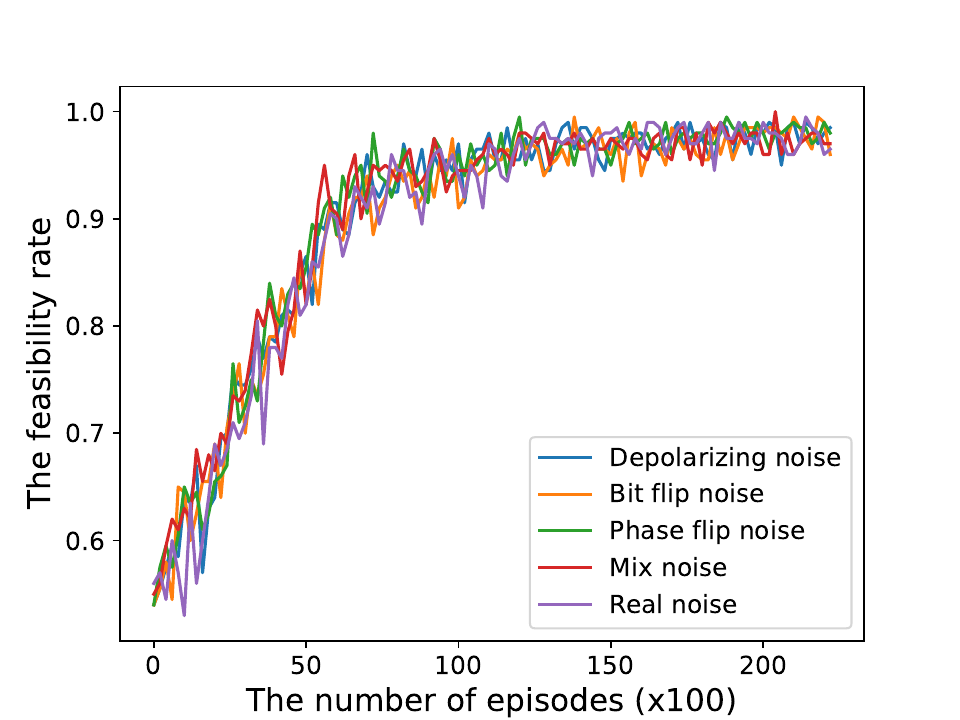}
  \caption{QML\_5}
  \label{}
\end{subfigure}
\hfill
\begin{subfigure}{.24\textwidth}
  \centering
  \includegraphics[width=1\textwidth]{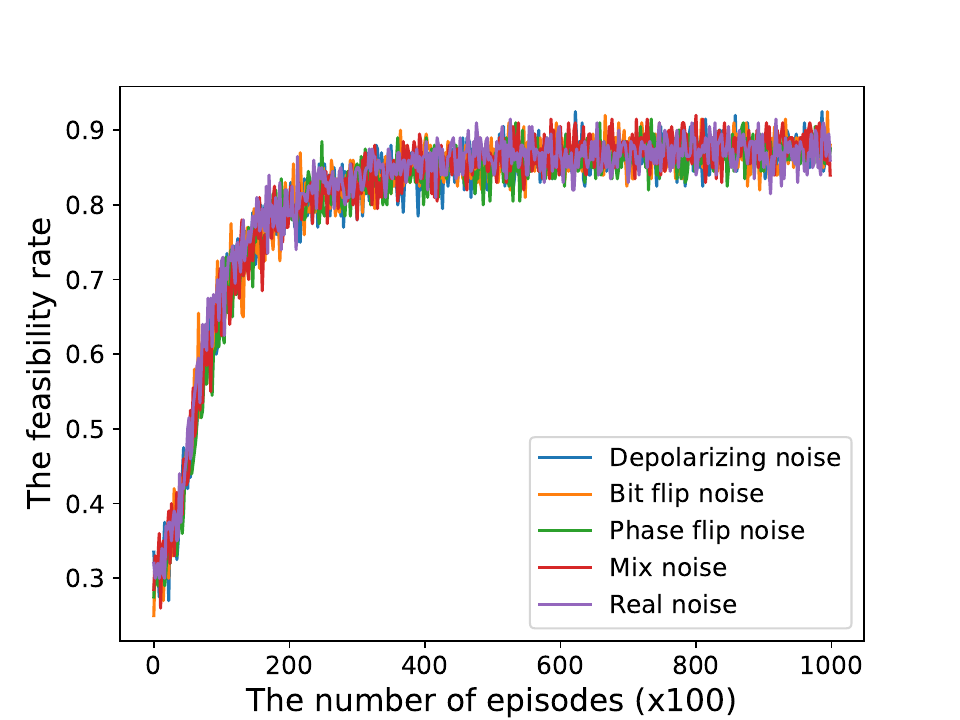}
  \caption{QML\_7}
  \label{}
\end{subfigure}
\caption{The feasibility rate during training with 4 datasets.} 
\label{fig:training:feasibility}
\end{figure*}

\begin{figure*}[ht]
\begin{subfigure}{.24\textwidth}
  \centering
  \includegraphics[width=1\textwidth]{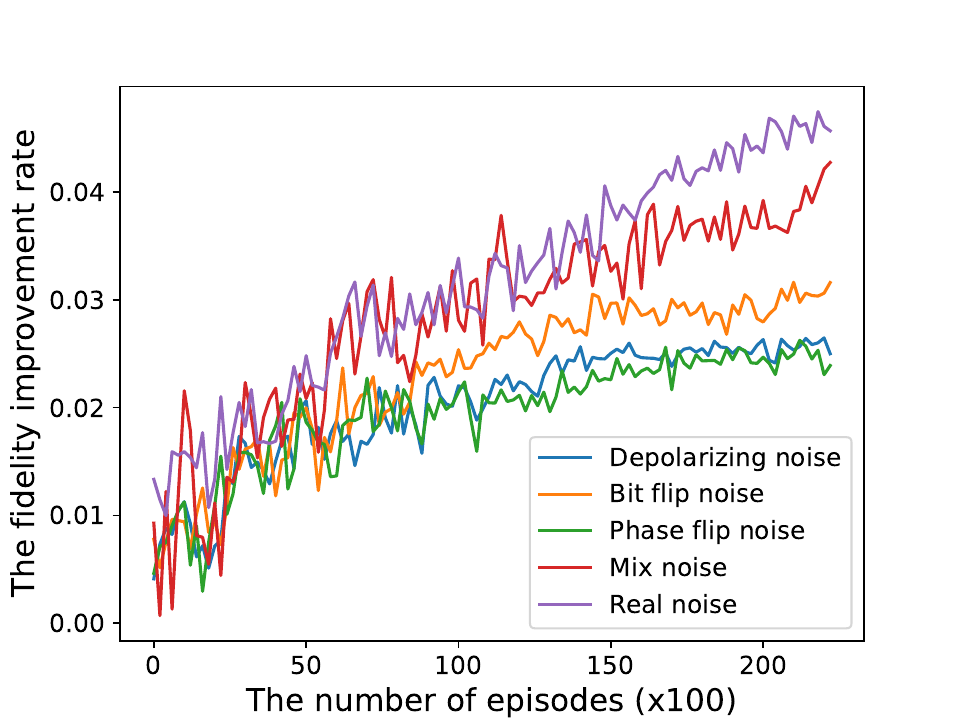}
  \caption{QAOA\_5}
  \label{}
\end{subfigure}
\hfill
\begin{subfigure}{.24\textwidth}
  \centering
  \includegraphics[width=1\textwidth]{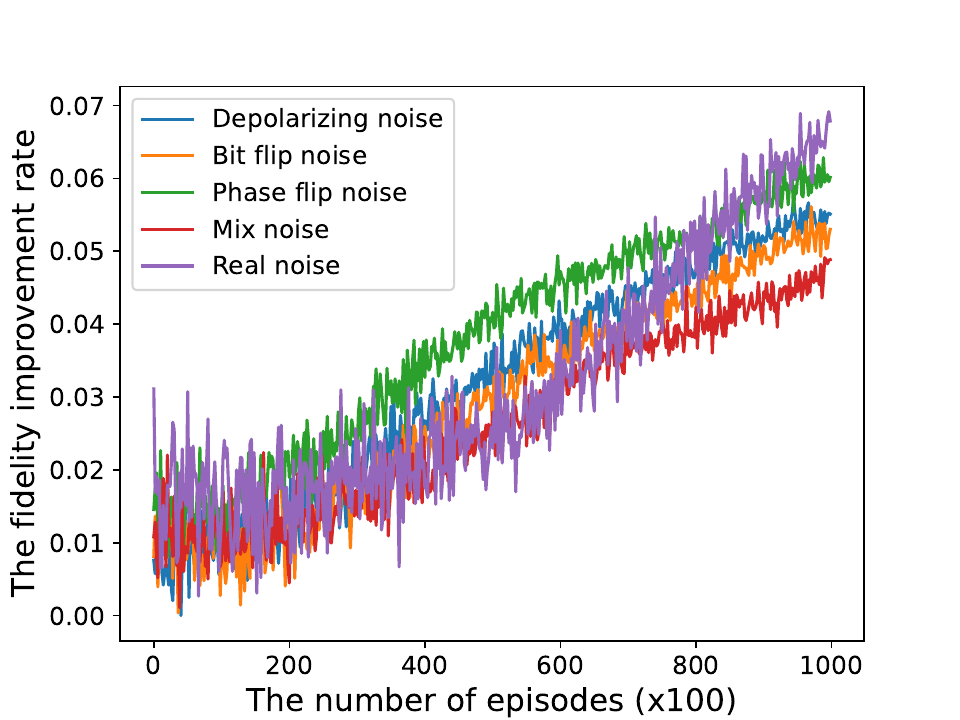}
  \caption{QAOA\_7}
  \label{}
\end{subfigure}
\hfill
\begin{subfigure}{.24\textwidth}
  \centering
  \includegraphics[width=1\textwidth]{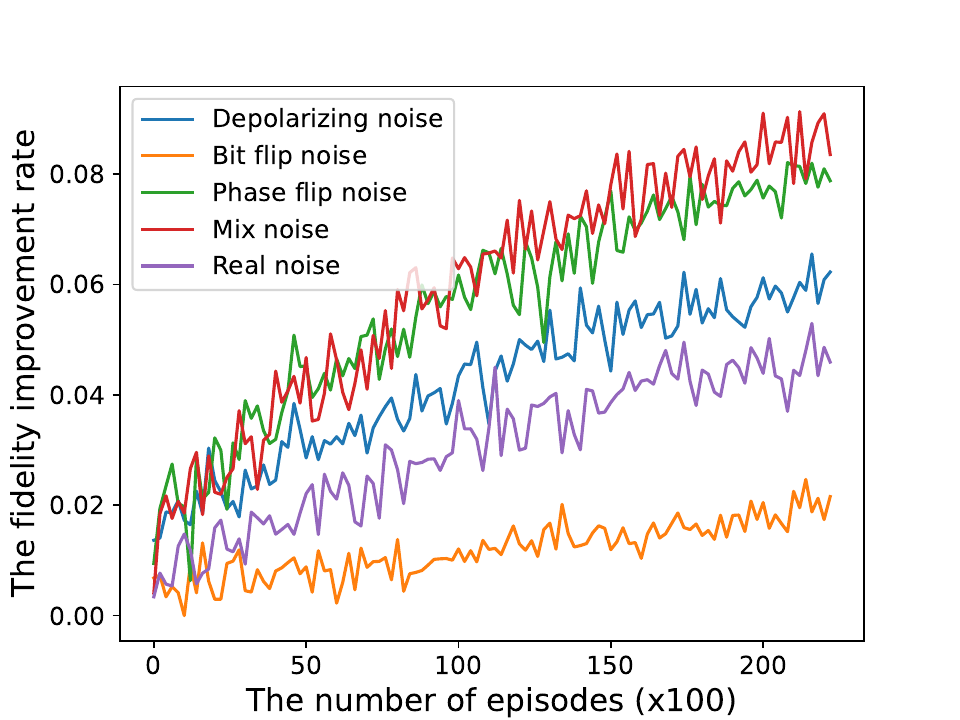}
  \caption{QML\_5}
  \label{}
\end{subfigure}
\hfill
\begin{subfigure}{.24\textwidth}
  \centering
  \includegraphics[width=1\textwidth]{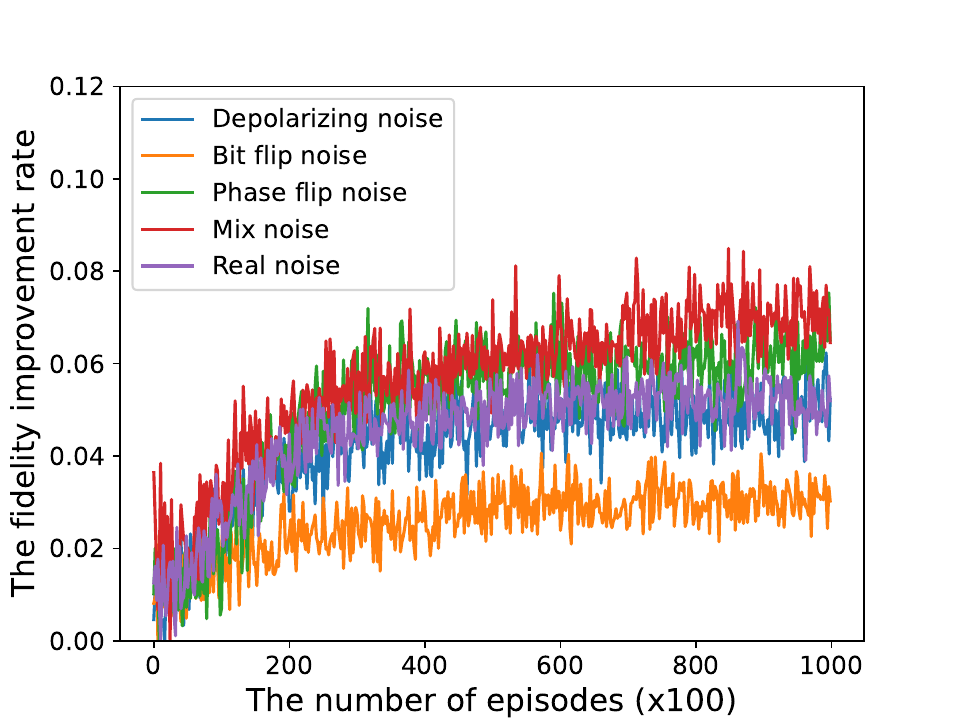}
  \caption{QML\_7}
  \label{}
\end{subfigure}
\caption{The fidelity improvement rate during training with 4 datasets.} 
\label{fig:training:fidelity}
\end{figure*}

\subsubsection{Inference performance}
We compare the fidelity of solutions produced by \NAMEA\ with those generated by state-of-the-art methods, including Qiskit, VIC and GNN-RL, across four test datasets: QAOA\_5, QAOA\_7, QML\_5 and QML\_7. We note that VIC is applicable only to QAOA circuits. Therefore, we include comparisons with VIC only for the QAOA\_5 and QAOA\_7 datasets. Through these comparisons, we demonstrate the practical efficiency of \NAMEA\ in optimizing fidelity.

Comparisons using the QAOA datasets, including QAOA\_5 and QAOA\_7, are aimed at evaluating \NAMEA's performance under scenarios with a small number of dependencies among target gates. Figure~\ref{fig:testing:fidelity:QAOA_5} compares the performance of \NAMEA\ with other methods for the QAOA\_5 dataset. The results show that \NAMEA\ consistently outperforms all baselines in nearly all cases, regardless of the number of target gates or the noise model applied. Specifically, when compared to VIC, a method specialized in minimizing the depth of QAOA circuits, \NAMEA\ achieves higher fidelity improvements of $1.5\%$, $1.5\%$, $1.4\%$, $2.8\%$ and $1.8\%$ under depolarizing, bit-flip, phase-flip, mixed, and real noise models, respectively. This highlights that directly optimizing fidelity is more effective for improving circuit reliability than focusing on circuit depth reduction.

Furthermore, when compared to Qiskit, a widely used commercial tool for circuit transpilation, \NAMEA\ demonstrates even greater fidelity gains of $1.1\%$, $0.05\%$, $1.6\%$, $1.7\%$ and $3.6\%$ under the same noise models. Notably, the fidelity improvement achieved by \NAMEA\ over Qiskit becomes increasingly significant when the number of target gates grows. This observation highlights \NAMEA's capability to maintain high reliability in large-scale quantum circuit transpilation, a critical requirement as quantum computing continues to scale and tackle more complex problems.

Compared to GNN-RL, which employs a GNN-based fidelity estimator, we can observe how improvements in fidelity estimation influence overall performance. Specifically, \mbox{\NAMEA} achieves even greater fidelity gains of 5.5\%, 2.8\%, 6.6\%, 8.0\%, and 6.7\% under the same noise models. These results suggest that the GP-based surrogate model can lead to better overall performance compared to the GNN-based estimator.

Figure~\ref{fig:testing:fidelity:QAOA_7} presents a comparison of \NAMEA\ with Qiskit, VIC, and GNN-RL on the more complex QAOA\_7 dataset. The results demonstrate that \NAMEA’s performance advantage becomes even more pronounced for this dataset. Specifically, \mbox{\NAMEA} achieves fidelity improvements of $4.2\%$, $2.7\%$, $2.9\%$, $2.5\%$ and $5.8\%$ over VIC; $4.2\%$, $0.9\%$, $3.8\%$, $2.0\%$ and $6.5\%$ over Qiskit; and $3.6\%$, $3.8\%$, $3.0\%$, $3.0\%$ and $10.1\%$ over GNN-RL under the depolarizing, bit-flip, phase-flip, mixed, and real noise models, respectively. These results reaffirm the effectiveness of \NAMEA\ in generating highly reliable circuits for larger-scale quantum computations. An additional observation is that the fidelity improvement is particularly significant under the real noise model. This demonstrates \NAMEA’s robustness and adaptability to real-world quantum noise conditions.

Comparisons using the QML datasets, including QML\_5 and QML\_7, are aimed at evaluating \NAMEA's performance under scenarios with a large number of dependencies among target gates. Figure~\ref{fig:testing:fidelity:QML_5} compares \NAMEA\ with Qiskit and GNN-RL on the QML\_5 dataset. Despite the increased complexity of target gate dependencies, \NAMEA\ consistently outperforms both baselines. Specifically, \mbox{\NAMEA} achieves fidelity improvements of $2.9\%$, $1.0\%$, $4.8\%$, $5.6\%$ and $6.5\%$ compared to Qiskit and $2.0\%$, $2.5\%$, $2.8\%$, $4.1\%$ and $2.1\%$ compared to GNN-RL under the depolarizing, bit-flip, phase-flip, mixed, and real noise models, respectively. The performance gains of \NAMEA\ are even more significant in the 7-qubit system, as shown in Figure~\ref{fig:testing:fidelity:QML_7}. Here, \mbox{\NAMEA} achieves fidelity improvements of $4.1\%$, $1.7\%$, $6.1\%$, $5.7\%$, and $7.6\%$ over Qiskit, and $12.1\%$, $7.6\%$, $12.8\%$, $12.6\%$, and $9.3\%$ over GNN-RL under the same noise models.

\begin{figure*}[ht]
  \centering
  \includegraphics[width=1\textwidth]{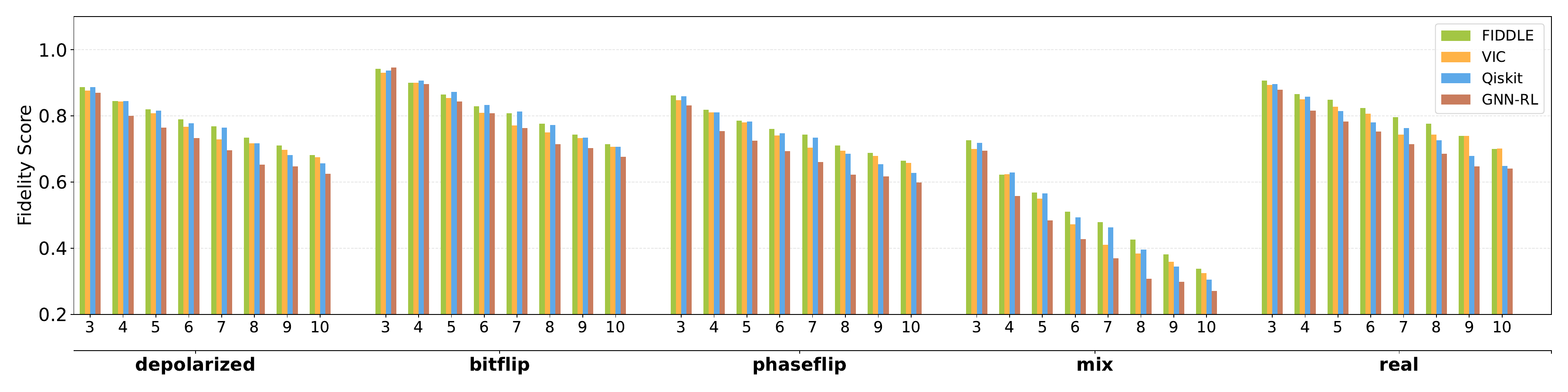}
  \caption{The comparison results for QAOA\_5.}
  \label{fig:testing:fidelity:QAOA_5}
\end{figure*}

\begin{figure*}[ht]
  \centering
  \includegraphics[width=1\textwidth]{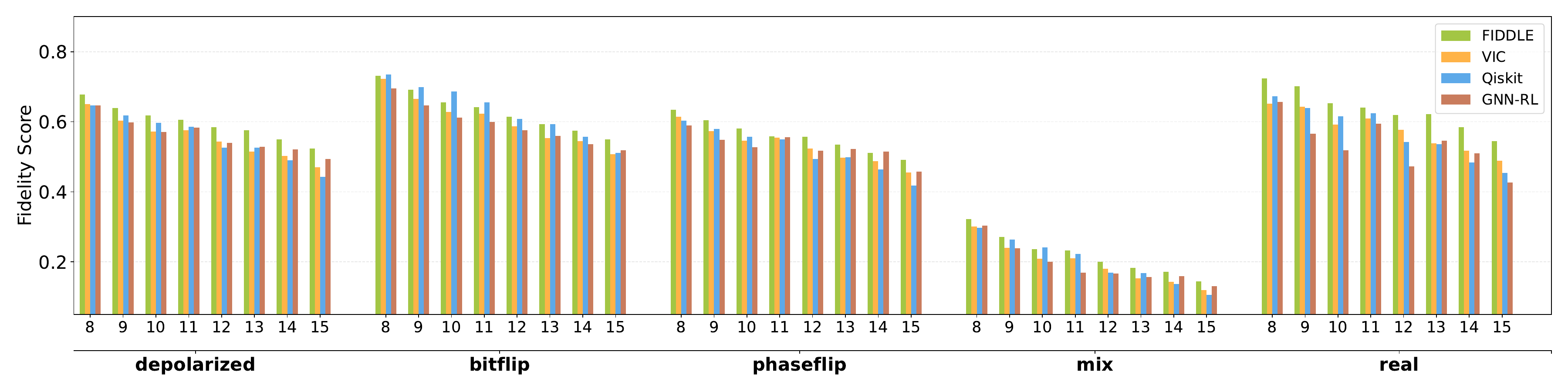}
  \caption{The comparison results for QAOA\_7.}
  \label{fig:testing:fidelity:QAOA_7}
\end{figure*}

\begin{figure*}[ht]
  \centering
  \includegraphics[width=1\textwidth]{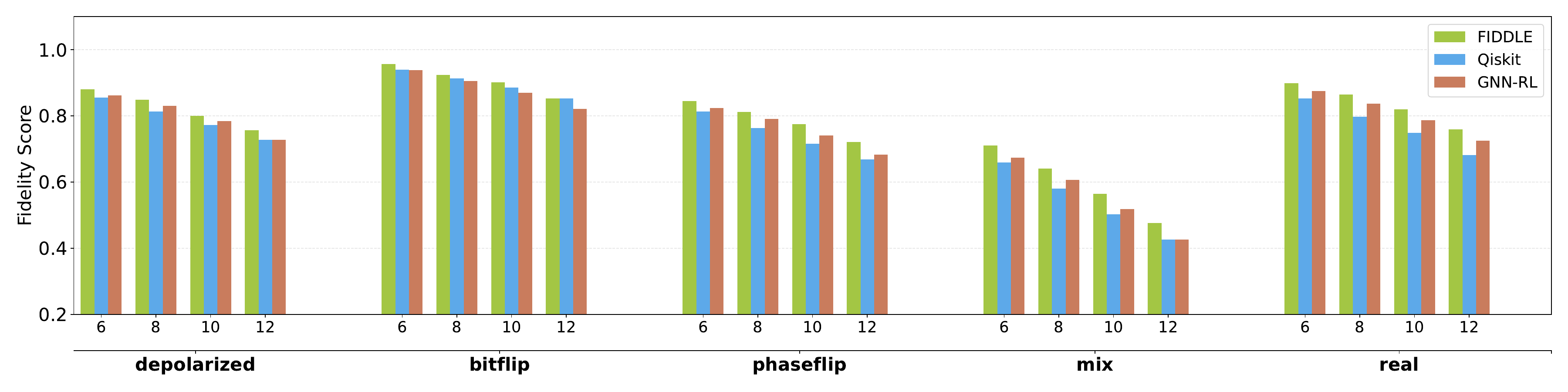}
  \caption{The comparison results for QML\_5.}
  \label{fig:testing:fidelity:QML_5}
\end{figure*}

\begin{figure*}[ht]
  \centering
  \includegraphics[width=1\textwidth]{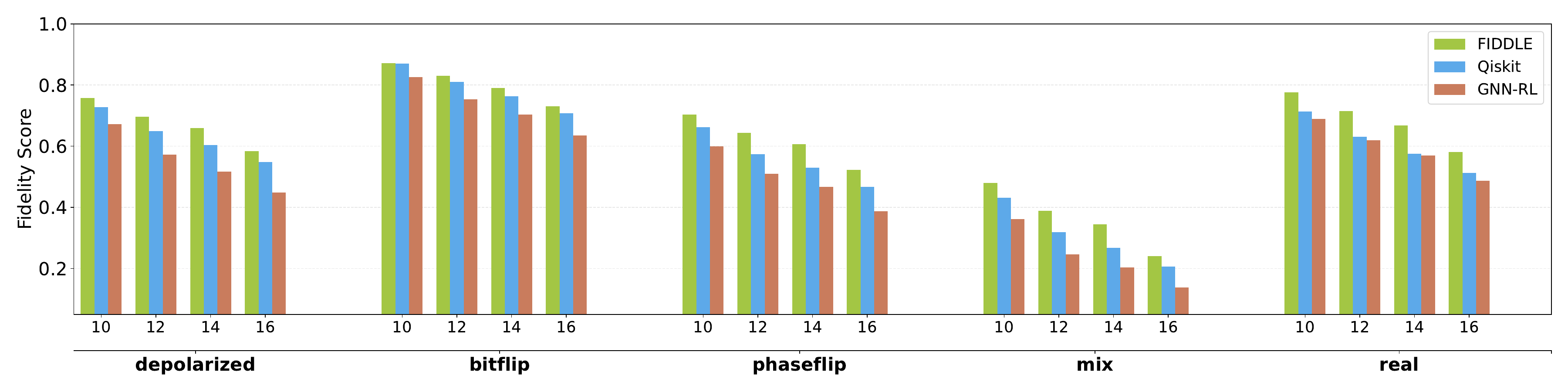}
  \caption{The comparison results for QML\_7.}
  \label{fig:testing:fidelity:QML_7}
\end{figure*}

\section{Conclusion}
\label{sec:conclusion}

In this work, we propose a learning-based framework, \NAMEA, to address two critical challenges of enhancing the reliability of quantum circuits in the NISQ era. In particular, our framework leverages a GP-based surrogate model to estimate the fidelity efficiently, even with limited training data, and an RL module to optimize gate sequences for fidelity. Rigorous evaluations against state-of-the-art fidelity estimation and routing optimization methods demonstrate that \NAMEA~ significantly improves the fidelity estimation and fidelity optimization across various noise models. As a result, this work lays the foundation for advancing gate-based quantum computing towards practical, scalable applications in the NISQ era and beyond.

\section*{Acknowledgements}
This work is supported by National Science Foundation (NSF) under award 2416606.

\bibliographystyle{ACM-Reference-Format}
\bibliography{bibliography}




\end{document}